%% file: sleepingbandits.tex
\documentclass{article}
\usepackage[preprint]{neurips_2023}

\usepackage[utf8]{inputenc} 
\usepackage[T1]{fontenc}    
\usepackage{url}            
\usepackage{booktabs}       
\usepackage{amsfonts}       
\usepackage{nicefrac}       
\usepackage{microtype}      
\usepackage{xcolor}         

\usepackage[algo2e,ruled]{algorithm2e}
\usepackage{amsfonts}
\usepackage{amsmath, bm, amsthm}
\usepackage{commath}
\usepackage{mathtools}
\usepackage[bb=dsserif]{mathalpha}
\usepackage{tikz}
\usetikzlibrary{shapes,arrows,positioning}
\usepackage{thm-restate}

\usepackage[colorlinks,linkcolor=blue,citecolor=blue,urlcolor=magenta,linktocpage,plainpages=false]{hyperref}

\newtheorem{theorem}{Theorem}
\newtheorem{remark}[theorem]{Remark}

\newtheorem{lemma}[theorem]{Lemma}

\newtheorem{corollary}[theorem]{Corollary}

\input{math_commands.tex}
\newcommand\scalemath[2]{\scalebox{#1}{\mbox{\ensuremath{\displaystyle #2}}}}
\newcommand{\Qchanged}[1]{#1} 

\title{Near-optimal Per-Action Regret Bounds for Sleeping Bandits}

\author{%
  Quan Nguyen, Nishant A. Mehta \\
  Department of Computer Science\\
  University of Victoria\\
  \texttt{manhquan233@gmail.com, nmehta@uvic.ca} \\
}

\begin{document}

\maketitle

\begin{abstract}
  We derive near-optimal per-action regret bounds for sleeping bandits, in which both the sets of available arms and their losses in every round are chosen by an adversary. 
  In a setting with $K$ total arms and at most $A$ available arms in each round over $T$ rounds, the best known upper bound is $O(K\sqrt{TA\ln{K}})$, obtained indirectly via minimizing internal sleeping regrets. Compared to the minimax $\Omega(\sqrt{TA})$ lower bound, this upper bound contains an extra multiplicative factor of $K\ln{K}$.
  We address this gap by directly minimizing the per-action regret using generalized versions of EXP3, EXP3-IX and FTRL with Tsallis entropy, thereby obtaining near-optimal bounds of order $O(\sqrt{TA\ln{K}})$ and $O(\sqrt{T\sqrt{AK}})$.
  We extend our results to the setting of bandits with advice from sleeping experts, generalizing EXP4 along the way. This leads to new proofs for a number of existing adaptive and tracking regret bounds for standard non-sleeping bandits.  
  Extending our results to the bandit version of experts that report their confidences leads to new bounds for the confidence regret that depends primarily on the sum of experts' confidences. 
 We prove a lower bound, showing that for any minimax optimal algorithms, there exists an action whose regret is sublinear in $T$ but linear in the number of its active rounds.
\end{abstract}

\section{INTRODUCTION}
The multi-armed bandit (MAB) framework and its variants have been widely used for practical applications in various domains such as clinical trials, finance and recommender systems~\citep[see e.g.][]{Bouneffouf2020MABApplication}. In the standard MAB framework, a learner interacts with $K$ arms over $T$ rounds. In each round, the learner chooses to observe the loss of one of the arms. While the losses of the arms in each round are unknown to the learner, the number of arms $K$ is assumed to be fixed in every round. However, this assumption does not always hold in practice. For example, in drug testing where each arm is a drug type, certain types of drugs can only be tested in some certain rounds, or new and more effective drugs might be available only in later rounds. In such scenarios, it is important to have algorithms capable of learning with time-varying sets of available arms. This is the \textit{sleeping bandits} setting~\citep{Kleinberg2010Sleeping}, where in each round $t=1,\dots,T$, only a subset $\sA_t \subseteq \{1,\dots,K\}$ of active arms are accessible to the learner.

The \textit{sleeping experts} setting (also known as the specialist setting)~\citep{Blum1997CalendarScheduling,Freund1997} is the full-information feedback variant of this problem, in which the losses of the active arms are revealed at the end of each round. 
Prior works on sleeping experts have mainly used two different notions of regret to measure the performance of a learner, namely per-action regret~\citep{BlumAndMansour2007a,Gaillard2014SecondOrderBoundExcessLosses,Luo2015AdaNormalHedge} and ordering regret~\citep{Kleinberg2010Sleeping,Kanade2014SleepingExperts, Neu2014CombinatorialSleepingPolicyRegret}. 
Besides the notions of regrets, an important characteristic of the setting is the stochastic or adversarial nature of the sets $\sA_t$ and the arms' losses.~\cite{Kanade2014SleepingExperts} indicated that obtaining a sublinear ordering regret bound is computationally hard when both $\sA_t$ and losses are adversarial. 
As a result, subsequent works on sleeping bandits usually assume at least one component to be stochastic~\citep{Slivkins2013,Neu2014CombinatorialSleepingPolicyRegret,Slivkins2014, Saha2020}. 
Recently,~\cite{Gaillard2023OneArrowTwoKills} developed a new notion of regret for sleeping bandits called \textit{sleeping internal regret}, which can be minimized efficiently in the fully adversarial setting with adversarial $\sA_t$ and adversarial losses.

Our work focuses on minimizing the per-action regret in the fully adversarial setting. This notion of regret compares the cumulative loss of the learner to that of a single best arm in hindsight during the rounds in which that arm was active. To the best of our knowledge, in the fully adversarial setting, no prior work has focused on directly deriving optimal per-action regret bounds. We are interested in obtaining more fine-grained bounds that depend on the maximum number of active arms in any round $A$, where $A = \max_{t=1,\dots,T}\abs{\sA_t} \leq K$. The smallest existing bound is the $O(K\sqrt{TA\ln{K}})$ bound by~\cite{Gaillard2023OneArrowTwoKills}, obtained indirectly from minimizing the internal sleeping regret. This bound can be much larger than an $\Omega(\sqrt{TA})$ minimax lower bound implied by suitably adapting an existing minimax lower bound construction for standard bandits~\citep{EXP3Auer2002b}. Moreover, as we show in this work, the factor of $K$ outside the square root can be eliminated entirely.

Another motivation for bounding the per-action regret in sleeping bandits is its implication on the adaptive and tracking regrets in standard non-sleeping bandits.
Adaptive regret (also known as interval regret)~\citep{HazanAdaptiveRegret2009, Luo2018ContextualNonstationary} is the regret against a fixed arm on a time interval, while tracking regret (also known as shifting or switching regret)~\citep{Herbster1998FixedShare} is the regret against a sequence of arms over $T$ rounds.
Previous work obtained adaptive and tracking regret bounds for standard non-sleeping experts via a reduction to regret bounds for sleeping experts~\citep{Freund1997SpecializedExperts,Adamskiy2016}.
In bandits, instead of the reduction to sleeping bandits, $\tilde{O}(\sqrt{T})$ bounds\footnote{$\tilde{O}$ hides terms in $K$, number of switches $S$ and $\ln{T}$.} on tracking and adaptive regret have been obtained via Fixed Share~\citep{EXP3Auer2002b,Herbster1998FixedShare,Luo2018ContextualNonstationary}.
Our work shows that the reduction to sleeping bandits also leads to $\tilde{O}(\sqrt{T})$ adaptive and tracking bounds.
  
\subsection*{Overview of Main Results and Techniques}
\begin{table}[t]
  \caption{A Summary of Bounds on Per-Action Regret. Hyphens indicate bounds that are either not comparable to a per-action regret bound or unavailable.} \label{table:relatedworks}
  \centering
  \begin{tabular}{lcll}
  \textbf{Algorithms}  & \textbf{Fully Adversarial?} & \textbf{Pseudo-Regret} & \textbf{High-Probability} \\
  \hline \\
  AUER~\scalebox{0.9}{\citep{Kleinberg2010Sleeping}}          & No (stochastic losses) & $\sqrt{TK\ln{T}}$ & \centering{-} \tabularnewline
  Sleeping-EXP3~\scalebox{0.9}{\citep{Saha2020}}             & No (stochastic $\sA_t$) & \centering{-} & \centering{-} \tabularnewline
  SR\_MAB~\scalebox{0.85}{\citep{BlumAndMansour2007a}}             & Yes & $K^2\sqrt{TA\ln{K}}$ & \centering{-} \tabularnewline
  SI-EXP3~\scalebox{0.9}{\citep{Gaillard2023OneArrowTwoKills}} & Yes & $K\sqrt{TA\ln{K}}$ & \centering{-} \tabularnewline
  \hline \\
  SB-EXP3 (this work) & Yes & $\sqrt{TA\ln{K}}$ & $\sqrt{TA\ln(K/\delta)}$\\
  FTARL (this work) & Yes & $\sqrt{T\sqrt{AK}}$ & $\scalemath{0.8}{\sqrt{T\sqrt{AK}} + \Qchanged{\sqrt{TA\ln\left(\frac{K}{\delta}\right)}}}$ 
  \end{tabular}
  \end{table}
We extend the EXP3~\citep{EXP3Auer2002b}, EXP3-IX~\citep{Neu2015ExploreNM}, Follow-The-Regularized-Leader (FTRL) with Tsallis entropy~\citep{Audibert2009minimax, Abernethy2015Fighting} and EXP4~\citep{EXP3Auer2002b} algorithms for standard bandits to sleeping bandits, obtaining new bounds that strictly generalize the existing bounds.
Our results lead to new proofs for $\tilde{O}(\sqrt{T})$ adaptive and tracking regret bounds for standard bandits. The generalized algorithms and analyses are adapted to the bandit-feedback version of the experts that report their confidences setting~\citep{BlumAndMansour2007a}. 
A summary of our contributions in comparison to prior works is in Table~\ref{table:relatedworks}.
All of our results hold for both pseudo-regret and high probability regret bounds.
Our paper is organized as follows (all proofs are in the appendix):
\begin{itemize}
  \item Section~\ref{sec:binaryTSF} introduces the $O(\sqrt{TA\ln{K}})$ and $O(\sqrt{T\sqrt{AK}})$ regret bounds for sleeping bandits.
  These bounds improve the best existing $O(K\sqrt{TA\ln{K}})$ bound, as well as recover the near-optimal $O(\sqrt{TK\ln{K}})$ and minimax $O(\sqrt{TK})$ bounds in non-sleeping bandits.
  Section~\ref{sec:SB-EXP3} shows a novel algorithm called SB-EXP3 and its $O(\sqrt{TA\ln{G_T}})$ regret bound guarantee for sleeping bandits, where $G_T \leq K$ is the number of arms that were active at least once after $T$ rounds.
  Its analysis relies on a new technique for bounding the growth of the potential function by decomposing the potential at round $t+1$ based on the set of active arms in round $t$. 
  In Section~\ref{sec:FTARL}, the $O(\sqrt{T\sqrt{AK}})$ bound is obtained by the Follow-the-Active-and-Regularized-Leader (FTARL) algorithm, an adaptation of FTRL with Tsallis entropy to sleeping bandits. Section~\ref{sec:RealValuedCases} considers the bandit-feedback version of the experts that report their confidences setting.
  Applying SB-EXP3 to this setting leads to new regret bounds which replace the dependence on $T$ and $A$ by the cumulative confidence over $T$ rounds.
  \item Section~\ref{sec:EXP4} studies the \textit{bandits with advice from sleeping experts} setting and presents SE-EXP4, a generalized version of EXP4 algorithm. 
  The analysis developed for SB-EXP3 also works for SE-EXP4, leading to the same $O(\sqrt{TK\ln{M}})$ regret bound of EXP4 with $M$ experts. 
  For standard bandits, this implies an $O(\sqrt{TK\ln(KT)})$ bound on adaptive regret. This bound is the same as the one obtained by~\cite{Luo2018ContextualNonstationary}, but with a different proof based on sleeping bandits instead of Fixed Share. This also implies both the $O(S\sqrt{KT\ln(KT)})$ and $O(\sqrt{SKT\ln(KT)})$ tracking regret bounds~\citep{EXP3Auer2002b,Neu2015ExploreNM} for unknown and known number of arm-switches $S$, respectively, where the latter is obtained via restarting SE-EXP4 after every $T/S$ rounds.

  \item Section~\ref{sec:AdaptiveLowerBound} defines the per-action strongly adaptive regret bound as a bound that depends only on $T_a$ for every action $a$, where $T_a$ is the number of active rounds of arm $a$. Extending the construction of~\cite{Daniely2015StronglyAdaptiveOL} for non-sleeping bandits to sleeping bandits, we show a linear $\Omega(T_a)$ per-action strongly adaptive lower bound. This implies that no algorithm can simultaneously guarantee an optimal per-action regret and sublinear $o(T_a)$ per-action regret for all arms.
\end{itemize}

\section{PRELIMINARIES}
\label{sec:setup}
We consider the adversarial multi-armed bandit problem with $K$ underlying arms, where $K$ might be unknown. 
Let $[K] = \{1,2,\dots,K\}$. In round $t = 1, 2, \dots, T$, a (possibly non-oblivious) adversary selects and reveals a set $\sA_t \subseteq [K]$ of active arms to the learner. Let $I_{i,t}=1$ (resp.~$I_{i,t}=0$) indicates that arm $i$ is active (resp.~inactive) in round $t$. Then, for each arm $i \in \sA_t$, the adversary selects a (hidden) loss value $\ell_{i,t} \in [0,1]$. The learner pulls one active arm $i_t \in \sA_t$ and observes loss $\ell_{i_t, t}$.

The learner's goal is to compete with the best arm in hindsight. 
For an arm $a \in [K]$, the regret of the learner with respect to arm $a$ is  the difference in the cumulative loss of the learner and that of arm $a$ over its active rounds:
\begin{align}
    R(a) = \sum_{t=1}^T I_{a,t}(\ell_{i_t,t}- \ell_{a,t}).
    \label{eq:regret}
\end{align}
We prove two types of regret bounds. The first is
\begin{align}
  \max_{a \in [K]}\E_{i_1, \dots, i_T}\left[R(a)\right] \leq \epsilon,
\end{align}
where the expectation is taken over the sequence of the learner's selected arms. In standard non-sleeping bandits, this corresponds to the notion of \textit{pseudo-regret}~\citep{EXP3Auer2002b}. If the adversary is oblivious, the 
pseudo-regret is equivalent to the expected regret.
The second type of bound is
\begin{align}
  \Pr\left(\max_{a \in [K]}R(a) \leq \epsilon\right) \geq 1-\delta,
\end{align}
where the probability is taken over the sequence of the learner's selected arms.

\textbf{Notations.} Let $A_t = \abs{\sA_t}$ be the number of active arms in round $t$ and $A = \max_{t \in [T]}A_t$ be the maximum value of $A_t$ over $T$ rounds. Let $\sG_t = \cup_{s=1,\dots,t}\sA_s$ be the set of arms that are active at least once in the first $t$ rounds. Let $G_t = \abs{\sG_t}$ be the size of $\sG_t$. We write $\hat{\ell}_t = \ell_{i_t,t}$ for the learner's loss in round $t$. Let $\Delta_n = \{p \in \R^n \mid p_i \geq 0, \sum_{i=1}^n p_i = 1\}$ be the $n$-dimensional probability simplex.

\begin{remark}
  The total number of underlying arms $K$ is fixed before learning, and the adversary cannot change $K$. On the other hand, $A_t$ and $G_t$ are decided by the adversary and vary over time. As some arms may never be active, $G_T$ can be strictly smaller than $K$.
\end{remark}

\section{NEAR-OPTIMAL REGRET UPPER BOUNDS}
\label{sec:binaryTSF}
In sleeping bandits, for any constant $A \in \{2,3,\dots,K\}$, there exists an $\Omega(\sqrt{TA})$ minimax pseudo-regret lower bound. 
The construction follows that of the minimax lower bound for standard bandits~\citep{EXP3Auer2002b} with $A$ arms always active and $K-A$ arms always inactive over $T$ rounds. 
In Section~\ref{sec:SB-EXP3}, we present SB-EXP3 (Algorithm~\ref{algo:SB-EXP3}) and its near-optimal $O(\sqrt{TA\ln{G_T}})$ pseudo-regret and high probability regret bounds. Note that SB-EXP3 does not require knowing $K$. In Section~\ref{sec:FTARL}, we show that when $K$ is known, an FTRL-based algorithm called FTARL (Algorithm~\ref{algo:FTARL})  obtains an $O(\sqrt{TA\ln{K}})$ bound with negative Shannon entropy and an $O(\sqrt{T\sqrt{AK}})$ bound with Tsallis entropy as the regularization function.

\subsection{Generalized EXP3 for Sleeping Bandits}
\label{sec:SB-EXP3}
In standard (non-sleeping) bandits, the EXP3 and EXP3-IX algorithms compute a distribution $p_t$ over arms in round $t$ based on the estimated regrets in previous rounds. 
Specifically, arms for which the estimated regrets are higher have larger probability of being sampled. 
In sleeping bandits, because arms can have different and even non-overlapping sets of rounds, it is unclear what kind of statistics about the arms should be maintained in each round. 
In particular, in any given round, the estimated regret for an arm might be  higher than for other arms due to it being active in more previous rounds, not because its losses are small in those rounds. 
Nevertheless, we will show that the estimated regret can be used effectively for selecting arms in each round.
\begin{algorithm2e}[t]
  \SetAlgoNoEnd
	\KwIn{$\eta > 0, \gamma \geq 0$}
	Initialize $\tilde{q}_{i, 1} = 1$ for $i = 1, 2, \dots, K$\; 
	\For{each round $t = 1, \dots, T$}{
        The adversary selects and reveals $\sA_t$\;	
        Compute $W_t = \sum_{i \in \sA_t} I_{i,t}\tilde{q}_{i,t}$\;       
        Compute $p_{i,t}$ by~\eqref{eq:pit}\;
		  Draw $i_t \sim p_t$ and observe $\hat{\ell}_t = \ell_{i_t, t}$\;
		  Compute loss estimate $\tilde{\ell}_{i, t}$ by~\eqref{eq:lossestimator}\;
		  Update $\tilde{q}_{i,t+1}$ by~\eqref{eq:updatestatistics}.
                
	}
	\caption{SB-EXP3 for sleeping bandits}
	\label{algo:SB-EXP3}
\end{algorithm2e}

Algorithm~\ref{algo:SB-EXP3} illustrates Sleeping Bandits using EXP3 (SB-EXP3), an adaptation of EXP3 and EXP3-IX to sleeping bandits. In round $t$, SB-EXP3 computes a probability vector $p_t \in \Delta_{G_t}$ over $\sG_t$. In the first round, $p_1$ is the uniform distribution. The learner samples $i_t \sim p_t$ and computes the loss estimates $\tilde{\ell}_t$ as follows:
\begin{align}
  \tilde{\ell}_{i,t} = \begin{cases}
      \frac{\ell_{i,t}\I{i_t = i}}{p_{i,t} + \gamma I_{i,t}} &\quad\text{for } i \in \sA_t,\\
      0 &\quad\text{for } i \in \sG_t \setminus \sA_t,
  \end{cases}
  \label{eq:lossestimator}
\end{align}
where $\gamma \geq 0$ is a parameter of the loss estimator. When $\gamma = 0$,~\eqref{eq:lossestimator} is equivalent to the unbiased loss estimate in EXP3. When $\gamma > 0$, due to $I_{i,t} = 1$ for all $i \in \sA_t$,~\eqref{eq:lossestimator} is the IX-loss estimator in EXP3-IX with exploration factor $\gamma$~\citep{Neu2015ExploreNM}. 
More generally, having different exploration factors for different arms may be beneficial. We explore such a case in Section~\ref{sec:RealValuedCases}.

The weight $\tilde{q}_{t+1}$ of arm $i$ at the beginning of round $t+1$ is defined as follows:
\begin{align}
    \tilde{q}_{i,t+1} = \exp\left(\eta\sum_{s=1}^{t}I_{i,s}(\ell_{i_s,s}-\tilde{\ell}_{i,s} - \gamma\sum_{j \in \sA_s} I_{j,s}\tilde{\ell}_{j,s})\right),
    \label{eq:updatestatistics}
\end{align}
where $\eta > 0$ is the learning rate.
The sampling probability of arm $i$ is proportional to its $\tilde{q}_{i,t}$, i.e., 
\begin{align}
  p_{i,t} = \begin{cases}
    \frac{\Qchanged{I_{i,t}}\tilde{q}_{i,t}}{W_t} &\quad\text{for } i \in \sA_t,\\
    0 &\quad\text{for } i \in \sG_t \setminus \sA_t,
\end{cases}
\label{eq:pit}
\end{align}
where $W_t = \sum_{k \in \sA_t}I_{k,t}\tilde{q}_{k,t}$ is the normalization factor.

Apart from setting zero sampling probability for inactive arms, similar to EXP3 and EXP3-IX, SB-EXP3 still follows the strategy of setting the sampling probability proportional to the exponential of the estimated per-action regret. 
The key difference is the added $-\gamma\sum_{j \in \sA_s}I_{j,s}\tilde{\ell}_{j,s}$ in the exponent, which is crucial for obtaining near-optimal high-probability bounds (where $\gamma > 0$). 
Intuitively, this term significantly reduces the sampling probability of arms that were frequently active in previous rounds but not frequently chosen.
Theorems~\ref{thm:inexpectationregretbound} and~\ref{thm:highprobregretbound} state the regret bounds of SB-EXP3.
\begin{restatable}{theorem}{TheoremInExpectationSBEXP}
    With $\gamma = 0$, for any $\eta > 0$, Algorithm~\ref{algo:SB-EXP3} guarantees
    \begin{align*}
      \max_{a \in [K]}\E[R(a)] &\leq \frac{\ln{G_T}}{\eta} + \frac{\eta}{2}\sum_{t=1}^T A_t.
  \end{align*}    
  Tuning $\eta$ leads to an $O\left(\sqrt{\ln(G_T)\sum_{t=1}^T A_t}\right)$ bound.
    \label{thm:inexpectationregretbound}
\end{restatable}

\begin{restatable}{theorem}{TheoremHighProbSBEXP}
  For any $\gamma \geq \frac{\eta}{2} > 0$, Algorithm~\ref{algo:SB-EXP3} guarantees
  \begin{align*}
      \max_{a \in [K]}R(a) &\leq \frac{\ln{G_T}}{\eta} + \frac{\ln(2G_T/\delta)}{\gamma} + \left(\frac{\eta}{2} + \gamma\right)\sum_{t=1}^T A_t
  \end{align*}
  with probability at least $1-\delta$.
  Tuning $\eta$ and $\gamma$ leads to an $O\left(\sqrt{\ln(G_T/\delta)\sum_{t=1}^T A_t}\right)$ bound.
  \label{thm:highprobregretbound}
\end{restatable}
Note that $\sum_{t=1}^T A_t \leq TA$. Since $A \leq G_T \leq K$, the bounds in Theorems~\ref{thm:inexpectationregretbound} and~\ref{thm:highprobregretbound} are generally smaller than the $O(\sqrt{TK\ln{K}})$ bounds of EXP3 and EXP3-IX. The difference is significant whenever $A \ll K$, which holds in many practical applications where the sets of active arms are sparse.

\textbf{Analysis Sketch}. The analyses of EXP3 and EXP3-IX~\citep{EXP3Auer2002b,Neu2015ExploreNM} treat the normalization factor $W_t$ as a potential function and bound the growth of $\frac{W_{t+1}}{W_t}$. This was possible in standard bandits because all $K$ arms are always active, hence $p_{i,t+1}$ can always be related to $p_{i,t}$. However, in sleeping bandits, the sets $\sA_t$ and $\sA_{t+1}$ might be non-overlapping, hence there might be no relationship between $W_{t+1}$ and $W_t$. Instead, we use 
\begin{align}
\tilde{Q}_t = \sum_{i \in \sG_T} \tilde{q}_{i,t}
\end{align}
as the potential function. The following key technical lemma bounds the growth of $\frac{\tilde{Q}_{t+1}}{\tilde{Q}_t}$.
\begin{restatable}{lemma}{LemmaUpperBoundQtSBEXP}
  For any $t \geq 0$,
  \begin{align*}
      \frac{\tilde{Q}_{t+1}}{\tilde{Q}_t} \leq \sum_{i \in \sA_t} p_{i,t}\exp\left(\eta(\hat{\ell}_t - \tilde{\ell}_{i,t} - \gamma\sum_{j \in \sA_t}\tilde{\ell}_{j,t})\right).
  \end{align*}
  \label{lemma:upperboundQt1Qt}
\end{restatable}
The proof of Lemma~\ref{lemma:upperboundQt1Qt} is based on decomposing $\tilde{Q}_{t+1}$ and $\tilde{Q}_t$ by $\sA_t$ and $\bar{\sA}_t = \sG_T \setminus \sA_t$. If arm $i \in \sA_t$, its weight $\tilde{q}_{i,t+1}$ at time $t+1$ can be related to $\tilde{q}_{i,t}$ via the update in~\eqref{eq:updatestatistics}. If arm $i \notin \sA_t$, by construction $\tilde{q}_{i,t+1} = \tilde{q}_{i,t}$. These two observations control the growth of $\tilde{Q}_{t+1}$ over $\tilde{Q}_t$. 

Lemma~\ref{lemma:upperboundQt1Qt} implies a bound on $\tilde{q}_{a,T+1}$ for any arm $a$ as
\begin{align*}
  \ln{\tilde{q}_{a, T+1}} &\leq \ln\tilde{Q}_{T+1} = \sum_{t=1}^T \ln{\frac{\tilde{Q}_{t+1}}{\tilde{Q}_t}}.
\end{align*}
Since $\tilde{q}_{a,T+1}$ grows with the estimated regret, this bound leads to an upper bound on the estimated regret, which in turns bounds the actual regret.

Observe that the dependency on $K$ in Algorithm~\ref{algo:SB-EXP3} can be removed completely: the initialization step of assigning $\tilde{q}$ to $1$ can be done implicitly. 
All other explicit computations only use the sets $\sA_t$ and $\sG_t$. 
As a result, SB-EXP3 is independent of $K$. This property is similar to that of the AdaNormalHedge algorithm~\citep{Luo2015AdaNormalHedge}, which obtains a low regret bound for sleeping experts when the total number of experts is unknown. 
Because the analysis of AdaNormalHedge relies on $[0,1]$-bounded loss vectors, it does not apply to sleeping bandits with the loss estimates in~\eqref{eq:lossestimator}.

\begin{remark}
  Lemma~\ref{lemma:upperboundQt1Qt} enables the proofs of both pseudo-regret and high-probability bounds without significant modifications to the algorithm. This is a major advantage over existing works in sleeping bandits, which provided only pseudo-regret bounds.
\end{remark}
In both Theorems~\ref{thm:inexpectationregretbound} and~\ref{thm:highprobregretbound}, optimally tuning the learning rate requires knowing $G_T$ and $\sum_{t=1}^T A_t$, which may not be available a priori. 
In Appendix~\ref{appendix:anytime}, we present Algorithm~\ref{algo:AnyTimeSBEXP3} which uses a two-level doubling trick to obtain a pseudo-regret bound of the same order without knowing these quantities beforehand.
\begin{restatable}{theorem}{TheoremAnytimeInExpectation}
    For any $T \geq 2$, Algorithm~\ref{algo:AnyTimeSBEXP3} (in Appendix~\ref{appendix:anytime}) guarantees that
    \begin{align*}
        \max_{a \in [K]}\E[R(a)] \leq \frac{4}{(\sqrt{2}-1)^2}\sqrt{\ln(G_T)\sum_{t=1}^T A_t}.
    \end{align*}
    \label{thm:anytimeinexpectationbound}
\end{restatable}
The same technique can be used to obtain a high-probability regret bound; however, the resulting bound is slightly larger (in the logarithmic term) due to the union bound.

\subsection{Generalized FTRL for Sleeping Bandits}
\label{sec:FTARL}
\begin{algorithm2e}[t]
  \SetAlgoNoEnd
	\KwIn{$\eta > 0, \gamma > 0, K \geq 2$}
	Initialize $\tilde{L}_{i, 0} = 0$ for $i = 1, 2, \dots, K$\; 
	\For{each round $t = 1, \dots, T$}{
        The adversary selects and reveals $\sA_t$\;
        Compute $q_t = \argmin_{q \in \Delta_K} \psi_t(q) + \inp{q}{\tilde{L}_{t-1}}$\;
        Compute $W_t = \sum_{i=1}^K I_{i,t}q_{i,t}$\;
        Compute $p_{i,t}$ by~\eqref{eq:pitbyqit}\;
        Draw $i_t \sim p_t$ and observe $\hat{\ell}_t = \ell_{i_t, t}$\;
        \For{each arm $i \in [K]$}{
          If $I_{i,t} = 1$, compute  $\tilde{\ell}_{i, t}$ by~\eqref{eq:lossestimator}\;
          If $I_{i,t} = 0$, compute $\tilde{\ell}_{i,t}$ by~\eqref{eq:litinactive}\;
        Update $\tilde{L}_{i,t} = \tilde{L}_{i,t-1} + \tilde{\ell}_{i,t}$;        
        }        
	}
	\caption{FTARL for sleeping bandits}
	\label{algo:FTARL}
\end{algorithm2e}
\Qchanged{In standard non-sleeping bandits, FTRL with Tsallis entropy obtains an $O(\sqrt{TK})$ minimax pseudo-regret bound, an $O(\sqrt{TK\ln(1/\delta)})$ high probability bound against an \emph{oblivious} adversary, and an $O(\sqrt{TK\ln(K/\delta)})$ high-probability bound against a \emph{non-oblivious} adversary~\citep{Audibert2009minimax,Luo2017CSCI699LectureNote13}. 
In sleeping bandits, under the assumption that $K$ is known, we will show that the Follow-the-Active-and-Regularized-Leader (FTARL) strategy in Algorithm~\ref{algo:FTARL} with $\frac{1}{2}$-Tsallis entropy obtains $O(\sqrt{T\sqrt{AK}})$ pseudo-regret. 
Against a \emph{non-oblivious} adversary, we further show that FTARL obtains an $O(\sqrt{T\sqrt{AK}} + \sqrt{TA\ln(K/\delta)})$ high-probability bound.
If the adversary is \emph{oblivious}, this bound is reduced to $O(\sqrt{T\sqrt{AK}} + \sqrt{TA\ln(1/\delta)})$.
Thus, when $A = K$, the bounds of FTARL recover those of FTRL on standard bandits.
}

In round $t$, Algorithm~\ref{algo:FTARL} computes the weight vectors $q_t \in \Delta_K$ using FTRL, i.e.,
\begin{align*}
q_t = \argmin_{q \in \Delta_K} \psi_t(q) + \inp{q}{\tilde{L}_{t-1}},
\end{align*}
where $\psi_t$ is the regularization function and  $\tilde{L}_{t} \in \R^K$ is the cumulative (estimated) loss vector of $K$ arms. In particular, $\tilde{L}_{i,t} = \sum_{s=1}^t \tilde{\ell}_{i,t}$. Note that this step is possible because $K$ and the simplex $\Delta_K$ are known.

While $q_t$ is a valid probability vector over $K$ arms, it cannot be used directly for sampling because inactive arms might have non-zero elements in $q_t$. The sampling probability vector $p_t$ is computed by taking elements in $\sA_t$ from $q_t$ and normalizing:
\begin{align}
  p_{i,t} = \frac{I_{i,t}q_{i,t}}{\sum_{i \in \sA_t}q_{i,t}}.  
  \label{eq:pitbyqit}
\end{align}
Similar to Algorithm~\ref{algo:SB-EXP3}, once an arm $i_t \sim p_t$ is sampled, the loss estimates of active arms are constructed by~\eqref{eq:lossestimator}. The key difference in Algorithm~\ref{algo:FTARL} compared to Algorithm~\ref{algo:SB-EXP3} is that the inactive arms have non-zero loss estimates. For an arm $i \notin \sA_t$,
\begin{align}
  \tilde{\ell}_{i,t} = \hat{\ell}_t - \gamma\sum_{j \in \sA_t}\tilde{\ell}_{j,t}.
  \label{eq:litinactive}
\end{align}
In Appendix~\ref{appendix:SBEXP3ver2}, we show that by having non-zero loss estimates for inactive arms as in~\eqref{eq:litinactive}, Algorithm~\ref{algo:FTARL} with negative Shannon entropy regularizer is equivalent to Algorithm~\ref{algo:SB-EXP3}. In general, the motivation behind~\eqref{eq:litinactive} is mostly technical: in every round, it ensures that the loss vector $\tilde{L}_t$ contains only non-negative values. This facilitates using a local norm analysis of standard FTRL~\citep[e.g.][]{OrabonaIntroToOnlineLearningBook}. For an unbiased loss estimator where $\gamma = 0$, using~\eqref{eq:litinactive} implies that the estimated regret is equivalent to 
\begin{align*}
  \sum_{t=1}^T I_{a,t}(\hat{\ell}_t - \tilde{\ell}_{a,t}) = \sum_{t=1}^T \left(\hat{\ell}_t - \tilde{\ell}_{a,t}\right),
\end{align*}
which resembles the regret of a standard non-sleeping experts problem with input loss vector $\tilde{\ell}$. A similar technique is used by~\cite{ChernovVovk2009ExpertEvaluatorsAdvice}.

Using the Tsallis entropy $\psi_t(x) = \frac{1}{\eta}\left(\frac{1-\sum_{i=1}^K x_i^\beta}{1-\beta}\right)$ with parameter $\beta$, we obtain the regret bounds of Algorithm~\ref{algo:FTARL} in Theorems~\ref{thm:FTARLinexpectationbound} and~\ref{thm:FTARLhighprobbound}.
\begin{restatable}{theorem}{TheoremInExpectationFTARL}
With $\gamma = 0$, for any $\beta \in (0,1), \eta > 0$, Algorithm~\ref{algo:FTARL} with Tsallis entropy guarantees
  \begin{align*}
      \max_{a \in [K]}\E[R(a)] \leq \frac{K^{1-\beta}}{\eta(1-\beta)} + \frac{\eta}{2\beta}TA^{\beta}.
  \end{align*}
  \label{thm:FTARLinexpectationbound}
Setting $\beta = \frac{1}{2}$ and tuning $\eta$ leads to an $O\left(\sqrt{2T\sqrt{AK}}\right)$ bound.
\end{restatable}

\begin{restatable}{theorem}{TheoremHighProbFTARL}
  For any $\beta, \gamma, \eta \in (0,1)$, Algorithm~\ref{algo:FTARL} 
  with Tsallis entropy guarantees
  \begin{align*}
    \scalemath{0.93}{
      R(a) \leq \frac{K^{1-\beta}}{\eta(1-\beta)} + \frac{\eta A^\beta T}{\beta} + \gamma AT + \left(\frac{\eta + \beta}{2\beta\gamma} + \frac{1}{2}\right)\ln\left(3/\delta\right) + \Qchanged{\frac{\ln(3K/\delta)}{2\gamma}}
    }
  \end{align*}    
  simultaneously for all $a \in [K]$ with probability at least $1-\delta$.
  Letting $\beta = \frac{1}{2}$ and tuning $\eta$ and $\gamma$ leads to an $O\left(\sqrt{T\sqrt{AK}} + \Qchanged{\sqrt{TA\ln(K/\delta)}}\right)$ bound.
  \label{thm:FTARLhighprobbound}
\end{restatable}
\begin{remark}
    If $G_T \leq K$ is known, then we can replace $K$ by $G_T$ in Algorithm~\ref{algo:FTARL} and in Theorems~\ref{thm:FTARLinexpectationbound} and~\ref{thm:FTARLhighprobbound}. 
    \Qchanged{Moreover, if the adversary is oblivious then the second term in the bound can be improved to $O(\sqrt{TA\ln(1/\delta)})$ and the bound becomes $O\left(\sqrt{T\sqrt{AK}} + \sqrt{TA\ln(1/\delta)}\right)$.}
\end{remark}  

\textbf{Analysis Sketch}. Similar to the analysis of Algorithm~\ref{algo:SB-EXP3}, the regret bound of Algorithm~\ref{algo:FTARL} is obtained by bounding the estimated regret $\sum_{t=1}^T \hat{\ell}_t - \tilde{\ell}_{a,t}$. 
By the local norm analysis of FTRL and properties of Tsallis entropy, we have
\begin{align*}
  \sum_{t=1}^T \hat{\ell}_t- \sum_{t=1}^T \tilde{\ell}_{a,t} \leq \frac{K^{1-\beta}}{\eta(1-\beta)} + \frac{1}{2}\sum_{t=1}^T \frac{\eta}{\beta}\sum_{i=1}^K \tilde{\ell}_{i,t}^2 q_{i,t}^{2-\beta}.
\end{align*}
\looseness=-1{In both cases $\gamma = 0$ and $\gamma > 0$, because inactive arms have non-zero elements in both $\tilde{\ell}_t$ and $\tilde{q}_t$, they contribute a non-zero positive amount in the last term on the right-hand side.
Furthermore, as the computation of $\tilde{\ell}_{i,t}$ uses $p_{i,t}$, we wish to replace $q_{i,t}$ by $p_{i,t}$.
The following technical lemma achieves this.}
\begin{restatable}{lemma}{LemmaQNonactiveArms}
    For any $t \geq 1$ and $\beta \in (0,1)$,
    \begin{align*}
        \sum_{i=1}^K \tilde{\ell}_{i,t}^2q_{i,t}^{2-\beta} \leq \sum_{i \in \sA_t}\tilde{\ell}_{i,t}^2 p_{i,t}^{2-\beta}.
    \end{align*}
    \label{lemma:upperboundLQofnonactivearms}
\end{restatable}
Then, we bound $\sum_{i \in \sA_t}\tilde{\ell}_{i,t}^2 p_{i,t}^{2-\beta} \leq A_t^\beta$ using standard tools such as Holder's inequality. Finally, either taking the expectation (when $\gamma = 0$) or using concentration inequalities of the IX-loss estimator~\citep{Neu2015ExploreNM} leads to the $O(\sqrt{T\sqrt{AK}})$ bound.
\begin{remark}
  The technique of assigning the observed loss to sleeping arms in~\eqref{eq:litinactive} is similar to the reduction from sleeping experts to standard prediction with expert advice~\citep{ChernovVovk2009ExpertEvaluatorsAdvice, Gaillard2014SecondOrderBoundExcessLosses}.
However, without Lemma~\ref{lemma:upperboundLQofnonactivearms}, this reduction by itself does not immediately imply Theorems~\ref{thm:FTARLinexpectationbound} and~\ref{thm:FTARLhighprobbound}. 
\end{remark}

\begin{remark}
  When $K$ is large compared to $A$ (i.e. sparse action sets), the $O\left(\sqrt{T\sqrt{AK}}\right)$ bound of FTARL can be much larger than the $O\left(\sqrt{TA\ln{G_T}}\right)$ bound of SB-EXP3. In contrast, when $A = \Theta(K)$ as in standard non-sleeping bandits, FTARL gives smaller bounds.
\end{remark}

\subsection{Bounds on Confidence Regret}
\label{sec:RealValuedCases}
To the adversary, selecting the set $\sA_t$ in round $t$ is equivalent to selecting $K$ binary values $I_{i,t} \in \{0,1\}$. 
We generalize the setting further by having the adversary select real-valued $I_{i,t} \in [0,1]$.
This new setting is the bandit feedback variant of the experts that report their confidences setting, in which $I_{i,t}$ is the confidence of expert $i$ in round $t$~\citep{BlumAndMansour2007a}.
In this case, $R(a)$ is the \textit{confidence regret} with respect to arm $a$~\citep{Gaillard2014SecondOrderBoundExcessLosses}.

More concretely, at the beginning of round $t$ the adversary selects and reveals $K$ real-valued $I_{i,t} \in [0,1]$. The set of active arms becomes $\sA_t = \{i \in [K] : I_{i,t} > 0\}$.
We apply Algorithm~\ref{algo:SB-EXP3} to this problem without any modifications: the computations of $\tilde{\ell}_{i,t}, \tilde{q}_{i,t}$ and $p_{i,t}$ are the same as in Equations~\eqref{eq:lossestimator},~\eqref{eq:updatestatistics} and~\eqref{eq:pit}. The full protocol and algorithm are given in Appendix~\ref{appendix:confidenceregret}. 
 We state the following high probability regret bound. 
\begin{restatable}{theorem}{TheoremHighProbRealCase}
  With optimally tuned $\eta$ and $\gamma$, Algorithm~\ref{algo:SB-EXP3} guarantees
  \begin{align*}
    R(a) \leq O\left(\sqrt{\sum_{t=1}^T\sum_{i=1}^K I_{i,t} \ln(G_T/\delta)}\right)
  \end{align*}
  simultaneously for all $a \in [K]$ with probability $1-\delta$.
  \label{thm:highprobregretboundReal}
\end{restatable}
Observe that in this setting, because $I_{i,t}$ can be different for different arms in $\sA_t$,  SB-EXP3 uses a different implicit exploration factor $\gamma_{i,t} = \gamma I_{i,t}$ in~\eqref{eq:lossestimator}. 
We call this novel strategy based on the arm-dependent IX-loss estimator \textit{confident implicit exploration}.
The proof of Theorem~\ref{thm:highprobregretboundReal} mostly follows that of Theorem~\ref{thm:highprobregretbound} with one added key insight: the concentration inequalities of the original IX-loss estimator also hold for the arm-dependent IX-loss estimator.
\begin{remark}
An $O\left(\sqrt{\sum_{t=1}^T\sum_{i=1}^K I_{i,t}\ln{G_T}}\right)$ pseudo-regret bound is obtained via a similar analysis. 
When $I_{i,t}$ are binary, since $\sum_{i=1}^K I_{i,t} = A_t \leq A$, these bounds recover the bounds in Theorems~\ref{thm:inexpectationregretbound} and~\ref{thm:highprobregretbound}. 
\end{remark}
\begin{remark}
    Similar to Theorem~\ref{thm:anytimeinexpectationbound}, the two-level doubling trick presented in Appendix~\ref{appendix:anytime} can be used on $\sum_{t=1}^T\sum_{i=1}^K I_{i,t}$ and $\ln(G_T)$ to obtain pseudo-regret and high-probability bounds of the same order (up to a logarithmic factor).
\end{remark}


\section{BANDITS WITH ADVICE FROM SLEEPING EXPERTS}
\label{sec:EXP4}
\begin{algorithm2e}[t]
  \SetAlgoNoEnd
	\KwIn{$\eta > 0, \gamma > 0$}
	Initialize $\tilde{q}_{m, 1} = 1$ for $m = 1, 2, \dots, M$\; 
	\For{each round $t = 1, \dots, T$}{
       An adversary selects and reveals $\sB_t$ and $E_t$\;	
        Compute $z_{m,t}$ by~\eqref{eq:zmt} and $p_{t} = z_tE_t$\;
		Draw $i_t \sim p_t$ and observe $\hat{\ell}_t = \ell_{i_t, t}$\;
		Compute $\tilde{\ell}_{k, t} = \frac{\ell_{k,t}\I{i_t = k}}{p_{k,t} + \gamma}$ for arms $k \in [K]$\;
      \For{each awake expert $m \in \sB_t$}{
        Construct loss estimate $\tilde{x}_{m,t} = \inp{E_{m,t}}{\tilde{\ell}_t}$\;
		Update $\tilde{q}_{m,t+1}$ by~\eqref{eq:qmt}.
      }        
	}
	\caption{SE-EXP4 for bandits with advice from sleeping experts}
	\label{algo:SE-EXP4}
\end{algorithm2e}
The bandits with expert advice framework~\citep{EXP3Auer2002b} considers the non-sleeping MAB where $M$ experts give advice to the learner in each round. 
We study a variant of this problem where a time-varying set $\sB_t \subseteq [M]$ of awake experts gives advice to the learner in round $t$. 
The advice of expert $m \in \sB_t$ is $E_{m,t} \in \Delta_K$.
The learner aims to compete with each expert during their active rounds.
More formally, let $I_{m,t} = 0$ and $I_{m,t} = 1$ denote whether expert $m$ is sleeping or awake in round $t$, respectively. Let $\ell_t$ be the (hidden) loss vector of $K$ arms in round $t$. The regret with respect to $m$ after $T$ rounds is 
\begin{align}
    R(m) = \sum_{t=1}^T I_{m,t}(\ell_{i_t,t} - \inp{E_{m,t}}{\ell_t}).
\end{align}
All theorems in this section are high probability bounds. 
The pseudo-regret bounds are in Appendix~\ref{appendix:SEEXP4}.

\subsection{Generalized EXP4} 
In the non-sleeping setting where all $M$ experts are always active, EXP4 and EXP4-IX~\citep{EXP3Auer2002b,Neu2015ExploreNM} obtain an $O(\sqrt{TK\ln{M}})$ pseudo-regret and $O(\sqrt{TK\ln(M/\delta)})$ high-probability bound, respectively.
We will show that SE-EXP4 (Algorithm~\ref{algo:SE-EXP4}) obtains these two bounds as well in the bandits with advice from sleeping experts setting.

Let $B_t = \abs{\sB_t}$ be the number of awake experts at round $t$.
Let $E_t$ be the $B_t \times K$ matrix whose columns are $(E_{m,t})_{m \in \sB_t}$.
In round $t$, SE-EXP4 samples an expert $m_t$ from a distribution $z_t \in \Delta_{B_t}$, then samples an arm $i_t \sim E_{m_t,t}$.  This is equivalent to sampling $i_t \sim p_t$ directly, where $p_t = z_tE_t$~\citep{BanditAlgorithmsBook2020}.  
The main idea of SE-EXP4 is to take the sleeping experts as ``augmented'' sleeping arms and apply SB-EXP3. In particular, the weight of expert $m$ is
\begin{align}
    \tilde{q}_{m,t} = \exp\left(\sum_{s=1}^{t-1}I_{m,s}\left(\hat{\ell}_s - \gamma\sum_{j=1}^K\tilde{\ell}_{j,s} - \tilde{x}_{m,s}\right)\right),
    \label{eq:qmt}
\end{align}
where $\tilde{x}_{m,s}$ is the estimated loss of expert $m$ and $\tilde{\ell}_{j,s}$ is the estimated loss of arm $j$ in round $s$. Initially, $\tilde{q}_{m,1}= 1$.
For an expert $m \in \sB_t$, its sampling probability $z_{m,t}$ is proportional to $\tilde{q}_{m,t}$, i.e.,
\begin{align}
  z_{m,t} = \frac{\tilde{q}_{m,t}}{\sum_{j \in \sB_t}\tilde{q}_{j,t}}.
  \label{eq:zmt}
\end{align}
Note that $z_{m,t} = 0$ for $m \notin \sB_t$. After $i_t \sim p_t$ is sampled, the loss estimates $\tilde{\ell}_t$ of $K$ arms are computed as in~\eqref{eq:lossestimator}.
Then, the losses of the awake experts are estimated by the inner product of their advice and $\tilde{\ell}_t$:
\begin{align}
    \tilde{x}_{m,t} = \inp{E_{m,t}}{\tilde{\ell}_t}.
    \label{eq:xmt}
\end{align}
Using the same analysis of SB-EXP3 implies:
\begin{restatable}{theorem}{TheoremHighProbSEEXP}
  For any $\gamma, \eta \in (0,1), \eta \leq 2\gamma$, SE-EXP4 guarantees 
  \begin{align}
      R(u) \leq \frac{\ln{M}}{\eta} + \frac{\ln(2M/\delta)}{2\gamma} + (\gamma + \frac{\eta}{2})TK + \ln(2/\delta)
  \end{align}
  simultaneously for all experts $u \in [M]$ with probability at least $1-\delta$, where the probability is taken over the sequence of the learner's selected arms. Tuning $\eta$ and $\gamma$ leads to an $O(\sqrt{TK\ln(M/\delta)})$ bound. 
  \label{thm:highprobregretboundEXP4}
\end{restatable}

\subsection{Adaptive and Tracking Bounds for Standard Adversarial Bandits}
Following~\cite{HazanAdaptiveRegret2009}, the adaptive regret of the learner on an interval $[t_1, t_2]$ with respect to arm $k$ in standard non-sleeping bandits is
\begin{align*}
    R_{[t_1, t_2]}(k) = \sum_{t=t_1}^{t_2} \left(\ell_{i_t,t} - \ell_{k,t}\right),
\end{align*}
  where $i_t$ is the arm chosen by the learner in round $t$.
To obtain an adaptive regret bound using SE-EXP4, we follow the ``virtual experts'' strategy similar to~\cite{Adamskiy2016}: for each interval $[t_1, t_2]$ and arm $k$, we create a virtual expert indexed by $(k,t_1, t_2)$ that is awake from round $t_1$ to round $t_2$ with advice $E_{(k,t), s} = e_k$ for any $s \geq t$, where $e_k$ is the $k^{th}$ vector in the standard basis of $\R^K$. 
There are $K{T \choose 2} = \frac{KT(T+1)}{2}$ such experts.
For any interval $[t_1,t_2]$, the regrets of experts $(1,t_1,t_2), (2, t_1,t_2), \dots, (K, t_1,t_2)$ are bounded by Theorem~\ref{thm:highprobregretboundEXP4}. 
This implies the following result.
\begin{restatable}{theorem}{TheoremAdaptiveBound}
For any $\gamma, \eta \in (0,1), \eta \leq 2\gamma$, SE-EXP4 with virtual experts guarantees that
\begin{align*}
    \scalemath{0.92}{R_{[t_1, t_2]}(k) \leq \frac{2\ln(KT)}{\eta} + \frac{\ln(KT/\delta)}{\gamma} + (\gamma + \frac{\eta}{2})TK + \ln(2/\delta)}
\end{align*} 
simultaneously for all intervals $[t_1, t_2]$ and arms $k \in [K]$ with probability $1-\delta$.
Tuning $\eta$ and $\gamma$ leads to an $O(\sqrt{TK\ln(KT/\delta)})$ bound. 
\label{thm:adaptiveBound}
\end{restatable} 
Next, we use Theorem~\ref{thm:adaptiveBound} to bound the tracking regret. The tracking regret of algorithm $\gA$ with respect to a sequence of arms $j_{1:T} = (j_1, j_2, \dots, j_T)$ is 
\begin{align*}
    R(j_{1:T}) = \sum_{t=1}^T (\hat{\ell}_t - \ell_{j_t,t}).
\end{align*}
Let $S = \sum_{t=2}^T \I{j_t \neq j_{t-1}}$ be the (unknown) number of switches in the sequence $j_{1:T}$. Since there are $S+1$ intervals of different competing arms, Theorem~\ref{thm:adaptiveBound} immediately implies an $O\left(S\sqrt{TK\ln(KT)}\right)$ tracking bound. Furthermore, the following corollary shows that if $S$ is known, then a tracking bound of order $O(\sqrt{STK\ln(\frac{KT}{\delta S})})$ is attainable as well. Thus, techniques based on sleeping bandits recover the tracking bounds in~\cite{EXP3Auer2002b} and~\cite{Neu2015ExploreNM}.
\begin{restatable}{corollary}{CorollaryTrackingBound}
    Restarting SE-EXP4 with virtual experts after every $T/S$ rounds guarantees that for any $j_{1:T}$ where $H(j_{1:T}) \leq S$, with probability at least $1-\delta$
    \begin{align*}
        R(j_{1:T}) \leq O\left(\sqrt{STK\ln\left(\frac{KT}{\delta S}\right)}\right).
    \end{align*}
    \label{corollary:sqrtStrackingbound}
\end{restatable}
\begin{remark}
  \cite{Luo2018ContextualNonstationary} obtained adaptive and tracking pseudo-regret bounds for contextual bandits which recover the same bounds (up to constants and logarithmic factors) in standard bandits. Our results hold for both pseudo-regret and high probability bounds. 
\end{remark}

\section{A PER-ACTION STRONGLY ADAPTIVE LOWER BOUND}
\label{sec:AdaptiveLowerBound}
Let $T_a = \abs{\{t \in [T]: I_{a,t}=1\}}$ be the number of rounds in which arm $a$ is active.
Prior works in sleeping experts established per-action regret bounds of order $O(\sqrt{T_a})$, independent of $T$~\citep[e.g.][]{ChernovVovk2010UnknownNoExperts,Gaillard2014SecondOrderBoundExcessLosses}. In this work, we call $T$-independent bounds that guarantee $R(a) = o(T_a)$ for all action $a \in [K]$ \textit{per-action strongly adaptive bounds}. 
The bounds obtained by SB-EXP3 and FTARL in Section~\ref{sec:binaryTSF} have $O(\sqrt{T})$ dependency for all arms $a$, thus they are not strongly adaptive. We study whether strongly adaptive bounds are achievable in sleeping bandits.
The following theorem shows a linear per-action strongly adaptive lower bound for a large class of algorithms that contains SB-EXP3, FTARL in Section~\ref{sec:binaryTSF} as well as any minimax optimal algorithm.

\begin{restatable}{theorem}{TheoremLowerBound}
    For any (possibly randomized) algorithm with guarantee 
    \begin{align}
        \sup_{a \in [K]}E[R(a)] \leq O(T^\gamma A^\beta(\ln(T))^\mu),
        \label{eq:TAlnT}
    \end{align}
    where $\gamma\in (0,1), \beta \geq 0, \mu \geq 0$ are constants, there exists a number of arms $K$ and sequence of sets of active arms and their losses such that $A=2$ and for at least one arm $a \in [K]$, 
    \begin{align*}
        T_a \geq \Omega(T^{1-\gamma}(\ln(T))^{-\mu}) \text{ and } \E[R(a)] \geq \Omega(T_a).
    \end{align*}
    \label{thm:lowerbound}
\end{restatable}
\begin{remark}
  By setting $\gamma = \beta = \frac{1}{2}, \mu = 0$,~\eqref{eq:TAlnT} corresponds to the $O(\sqrt{TA})$ bound of any minimax optimal algorithms. This implies that no algorithms can simultaneously have an optimal per-action regret bound and a sublinear per-action strongly adaptive bound.
\end{remark}
  Our proof extends that of~\cite{Daniely2015StronglyAdaptiveOL} to sleeping bandits. 
  Specifically, the setup contains arm $1$ that is always active with small losses, while the other $K-1$ arms are active in $K-1$ non-overlapping intervals with large losses, one interval for each arm. 
  To guarantee small regret against arm $1$, with high probability, the learner must pull only arm $1$ in the interval of some arm $j > 1$. 
Consider a slightly different setup where the losses of arm $j$ are smaller than that of arm $1$. Because the sequence of active arms and their losses in the two setups are identical from the first round until the start of the interval of arm $j$, the learner is unable to distinguish between the two setups and incurs linear regret against arm $j$ in the second setup.

  The limitation of this construction is that $K$ has to grow with $T$. 
  In particular, we require $K$ of order $T^{1-\gamma}2^{-\beta}(\ln(T))^{-\mu}$. 
  As a result, algorithms with bounds that are sublinear in $T$ but have large dependency on $K$, for example $O(T^\gamma K)$, do not satisfy~\eqref{eq:TAlnT}. Note that SB-EXP3 and FTARL always satisfy~\eqref{eq:TAlnT}, since $\sqrt{TA\ln{K}} \leq \sqrt{TA\ln{T}}$ and $\sqrt{T\sqrt{AK}} \leq T^{3/4}A^{1/4}$ for any $K \leq T$.

\section{CONCLUSION}
We derived algorithms for sleeping bandits and proved their near-optimal per-action regret bounds. 
These algorithms and their regret bounds strictly generalize existing approaches and results for standard non-sleeping bandits. 
We showed that sleeping bandits-based approaches both imply new bounds and recover a number of existing order-optimal $\tilde{O}(\sqrt{T})$ bounds in related settings with fundamentally different proofs. 
Furthermore, the analysis can be used to show both pseudo-regret and high probability bounds by using either the unbiased or IX-loss estimators.

A direction for future work is to either prove an $\Omega(\sqrt{TA\ln{G_T}})$ lower bound, or show that an $O(\sqrt{TA})$ upper bound is possible, thereby obtaining minimax optimal bounds.
For the former, such a lower bound must hold only in restricted conditions such as $A\ln{G_T} < K$, so there is no contradiction to the optimal $O(\sqrt{TK})$ bound in non-sleeping bandits.
The latter likely requires new analysis techniques other than the potential-based analysis and FTRL.

\begin{ack}
  This work was supported by the NSERC Discovery Grant RGPIN-2018-03942.
\end{ack}

\bibliographystyle{apalike}
\bibliography{sleepingbandits.bib}

\appendix
 \onecolumn
\section{PROOFS FOR SECTION~\ref{sec:SB-EXP3}}
Recall that $\tilde{Q}_t = \sum_{i \in \sG_T} \tilde{q}_{i,t}$.
For any set $S \subseteq \sG_T$, let $\tilde{Q}_{S,t} = \sum_{i \in S}\tilde{q}_{i,t}$ be the projection of $\tilde{Q}_t$ onto the set $S$. Let $\bar{S} = \sG_T \setminus S$ be the complement of $S$. Note that the normalization factor $W_t$ is equal to $\sum_{i \in \sA_t}\tilde{q}_{i,t} = \tilde{Q}_{\sA_t,t}$.

First, we prove a technical lemma which holds for any $I_{i,t} \in [0,1]$ for all $i \in [K]$ and $t \in [T]$.
\begin{lemma}
  For any $t \geq 1$,
  \begin{align*}
      \E_{i \sim p_t}[\tilde{\ell}_{i,t}] = \hat{\ell}_t - \gamma\sum_{j \in \sA_t}I_{j,t}\tilde{\ell}_{j,t}.
  \end{align*}
  \label{lemma:IXcomputeEipt}
\end{lemma}
\begin{proof}
  Using $\tilde{\ell}_{j,t} = 0$ for $j \neq i_t$, we obtain
  \begin{align*}
      \E_{i \sim p_t}[\tilde{\ell}_{i,t}] &=  \sum_{i \in \sA_t}p_{i,t}\tilde{\ell}_{i,t} \\
      &= p_{i_t, t}\tilde{\ell}_{i_t, t} \\
      &= p_{i_t, t}\frac{\hat{\ell}_t}{p_{i_t, t} + \gamma I_{i,t}} \\
      &= \hat{\ell}_t - \gamma I_{i_t,t} \frac{\hat{\ell}_t}{p_{i_t, t} + \gamma I_{i_t,t}} \\
      &= \hat{\ell}_t - \gamma I_{i_t,t}\tilde{\ell}_{i_t} \\
      &= \hat{\ell}_t - \gamma\sum_{j \in \sA_t}I_{j,t}\tilde{\ell}_{j,t}.
  \end{align*}
\end{proof}
\subsection{Proof of Lemma~\ref{lemma:upperboundQt1Qt}}
We make use of the following two facts that can be proved easily:
    \begin{itemize}
        \item Fact 1: the function $f(x) = e^{-\eta x}$ is convex for any $\eta \in \R$.
         \item Fact 2: For any $a, b > 0, c \geq 0$, if $a \geq b$ then 
        \begin{align*}
            \frac{a}{b} \geq \frac{a+c}{b+c}.
        \end{align*}
    \end{itemize}
    \LemmaUpperBoundQtSBEXP*
    \begin{proof}
    By Jensen's inequality and Fact 1, we have 
    \begin{align*}
        \sum_{i \in \sA_t} p_{i,t}\exp(-\eta\tilde{\ell}_{i,t}) &= \E_{i \sim p_t}[\exp(-\eta\tilde{\ell}_{i,t})]  \\
        &\geq \exp\left(-\eta\E_{i \sim p_t}[\tilde{\ell}_{i,t}]\right) \\
        &= \exp\left(-\eta\left(\hat{\ell}_t - \gamma\sum_{j \in \sA_t}\tilde{\ell}_{j,t}\right)\right),
    \end{align*}
    where the last equality is due to Lemma~\ref{lemma:IXcomputeEipt}.
    Multiplying by $\exp\left(\eta\left(\hat{\ell}_t - \gamma\sum_{j \in \sA_t}\tilde{\ell}_{j,t}\right)\right)\tilde{Q}_{A_t, t} > 0$ on both sides, we obtain 
    \begin{align}
        \sum_{i \in \sA_t} p_{i,t}\tilde{Q}_{\sA_t, t}\exp\left(\eta\left(\hat{\ell}_t-\tilde{\ell}_{i,t} - \gamma\sum_{j \in \sA_t}\tilde{\ell}_{j,t}\right)\right) \geq \tilde{Q}_{\sA_t, t}.
        \label{eq:IXpitQAtfirst}
    \end{align}  
    For all $i \in \sA_t$, we have $I_{i,t}=1$ and thus $p_{i,t} = \frac{\tilde{q}_{i,t}}{\tilde{Q}_{\sA_t, t}}$. Equation~\eqref{eq:IXpitQAtfirst} is equivalent to 
    \begin{align*}
        \sum_{i \in \sA_t} \tilde{q}_{i,t}\exp\left(\eta\left(\hat{\ell}_t-\tilde{\ell}_{i,t} - \gamma\sum_{j \in \sA_t}\tilde{\ell}_{j,t}\right)\right) \geq \tilde{Q}_{\sA_t, t}.
        \label{eq:IXpitQAtsecond}
    \end{align*}
    By our update rule, $\tilde{q}_{i,t+1} = \tilde{q}_{i,t}$ for $i \notin \sA_t$ and $\tilde{q}_{i,t+1} = \tilde{q}_{i,t}\exp\left(\eta\left(\hat{\ell}_t-\tilde{\ell}_{i,t} - \gamma\sum_{j \in \sA_t}\tilde{\ell}_{j,t}\right)\right)$ for $i \in \sA_t$. Hence,
    \begin{align*}
        \sum_{i \in \sA_t} \tilde{q}_{i,t+1} = \sum_{i \in \sA_t} \tilde{q}_{i,t}\exp\left(\eta\left(\hat{\ell}_t-\tilde{\ell}_{i,t} - \gamma\sum_{j \in \sA_t}\tilde{\ell}_{j,t}\right)\right) \geq \tilde{Q}_{\sA_t, t}.
    \end{align*}
    Applying Fact 2 for $a = \sum_{i \in \sA_t} \tilde{q}_{i,t+1}, b = \tilde{Q}_{\sA_t, t}$ and $c = \tilde{Q}_{\bar{\sA}_t, t}$, we obtain 
    \begin{nalign}
        \sum_{i \in \sA_t} p_{i,t}\exp\left(\eta\left(\hat{\ell}_t-\tilde{\ell}_{i,t} - \gamma\sum_{j \in \sA_t}\tilde{\ell}_{j,t}\right)\right) &= \frac{\sum_{i \in \sA_t} \tilde{q}_{i,t}\exp\left(\eta\left(\hat{\ell}_t-\tilde{\ell}_{i,t} - \gamma\sum_{j \in \sA_t}\tilde{\ell}_{j,t}\right)\right)}{\tilde{Q}_{\sA_t, t}} \\
        &= \frac{\sum_{i \in \sA_t} \tilde{q}_{i,t+1}}{\tilde{Q}_{\sA_t, t}} \\
         &\geq \frac{\sum_{i \in \sA_t} \tilde{q}_{i,t+1} + \tilde{Q}_{\bar{\sA}_t, t}}{\tilde{Q}_{\sA_t, t} + \tilde{Q}_{\bar{\sA}_t, t}} \\
        &= \frac{\tilde{Q}_{t+1}}{\tilde{Q}_t},
    \end{nalign}
    where the last equality is due to the fact that $\sum_{i \in \bar{\sA}_t}\tilde{q}_{i,t+1} = \sum_{i \in \bar{\sA}_t}\tilde{q}_{i,t} = \tilde{Q}_{\bar{\sA}_t, t}$.
\end{proof}
\subsection{Bounding the Estimated Regret}
Using Lemma~\ref{lemma:upperboundQt1Qt}, we bound the estimated regret as a function of the cumulative estimated losses of all active arms over $T$ rounds.
\begin{lemma}
    For any $\gamma \geq 0$, any arm $a \in [K]$,
    \begin{align*}
        \sum_{t=1}^T I_{a,t} \hat{\ell_t} \leq \frac{\ln{G_T}}{\eta} + \sum_{t=1}^T I_{a,t}\tilde{\ell}_{a,t} + \left(\frac{\eta}{2} + \gamma\right)\sum_{t=1}^T \sum_{i \in \sA_t} \tilde{\ell}_{i,t}.
    \end{align*}    
    \label{lemma:boundEstimatedRegretSB-EXP3}
\end{lemma}
\begin{proof}
    For any arm $a \in \sG_T$, we have 
    \begin{align}
        \ln{\tilde{Q}_{T+1}} &= \ln{\sum_{i \in \sG_T} \tilde{q}_{i,T+1}} \geq \ln{\tilde{q}_{a, T+1}} = \eta\sum_{t=1}^T I_{a,t}(\hat{\ell}_t - \tilde{\ell}_{a,t} - \gamma\sum_{j \in \sA_t}\tilde{\ell}_{j,t}).
        \label{eq:lowerBoundQT1}
    \end{align}
    On the other hand, we have
    \begin{align*}
        \ln{\tilde{Q}_{T+1}} &= \ln{\tilde{Q}_1} + \sum_{t=1}^T \ln{\frac{\tilde{Q}_{t+1}}{\tilde{Q}_t}} \\
        &\leq \ln{G_T} + \sum_{t=1}^T \ln\left(\sum_{i \in \sA_t} p_{i,t}\exp\left(\eta\left(\hat{\ell}_t - \tilde{\ell}_{i,t} - \gamma\sum_{j \in \sA_t}\tilde{\ell}_{j,t}\right)\right)\right) \\
        &= \ln{G_T} + \sum_{t=1}^T \ln\left(\exp\left(\eta\left(\hat{\ell}_t - \gamma\sum_{j \in \sA_t}\tilde{\ell}_{j,t}\right)\right)\sum_{i \in \sA_t}p_{i,t}\exp(-\eta\tilde{\ell}_{i,t})\right) \\
        &= \ln{G_T} + \sum_{t=1}^T \left( \eta\left(\hat{\ell}_t - \gamma\sum_{j \in \sA_t}\tilde{\ell}_{j,t}\right) + \ln\left(\sum_{i \in \sA_t}p_{i,t}\exp(-\eta\tilde{\ell}_{i,t})\right) \right) \\
        &\leq \ln{G_T} + \sum_{t=1}^T \left( \eta\left(\hat{\ell}_t - \gamma\sum_{j \in \sA_t}\tilde{\ell}_{j,t}\right) + \ln\left(\sum_{i \in \sA_t}p_{i,t}(1 + \frac{\eta^2\tilde{\ell}_{i,t}^2}{2} - \eta\tilde{\ell}_{i,t})\right) \right)  \\
        &= \ln{G_T} + \sum_{t=1}^T \left( \eta\left(\hat{\ell}_t - \gamma\sum_{j \in \sA_t}\tilde{\ell}_{j,t}\right) + \ln\left(1 + \eta^2\sum_{i \in \sA_t} \frac{p_{i,t}\tilde{\ell}_{i,t}^2}{2} - \eta \sum_{i \in \sA_t}p_{i,t}\tilde{\ell}_{i,t}\right) \right) \\
        &\leq \ln{G_T} + \sum_{t=1}^T \left( \eta\left(\hat{\ell}_t - \gamma\sum_{j \in \sA_t}\tilde{\ell}_{j,t}\right) + \eta^2\sum_{i \in \sA_t} \frac{p_{i,t}\tilde{\ell}_{i,t}^2}{2} - \eta \sum_{i \in \sA_t}p_{i,t}\tilde{\ell}_{i,t} \right) \\
        &= \ln{G_T} + \eta^2\sum_{t=1}^T \sum_{i \in \sA_t} \frac{p_{i,t}\tilde{\ell}_{i,t}^2}{2},
    \end{align*}
    where 
    \begin{itemize}
        \item the first inequality is due to Lemma~\ref{lemma:upperboundQt1Qt},
        \item the second inequality is $\exp(-x) \leq 1 + \frac{x^2}{2}-x$ for all $x \geq 0$,
        \item the third inequality is $\ln(1+x) \leq x$ for all $x \geq -1$,
        \item the last equality is due to Lemma~\ref{lemma:IXcomputeEipt}.
    \end{itemize}
    Since the losses are bounded in $[0,1]$ we also have $\sum_{i \in \sA_t} p_{i,t}\tilde{\ell}_{i,t}^2 \leq \sum_{i \in \sA_t} \tilde{\ell}_{i,t}$. This implies 
    \begin{align}
        \ln{\tilde{Q}_{T+1}} \leq \ln{G_T} + \frac{\eta^2}{2}\sum_{t=1}^T \sum_{i \in \sA_t} \tilde{\ell}_{i,t}.
        \label{eq:upperBoundQT1}
    \end{align}
    From~\eqref{eq:lowerBoundQT1} and~\eqref{eq:upperBoundQT1}, we obtain 
    \begin{align*}
        \sum_{t=1}^T I_{a,t}(\hat{\ell_t} - \tilde{\ell}_{a,t} - \gamma\sum_{j \in \sA_t}\tilde{\ell}_{j,t}) \leq \frac{\ln{G_T}}{\eta} + \frac{\eta}{2}\sum_{t=1}^T \sum_{i \in \sA_t} \tilde{\ell}_{i,t}.
    \end{align*}
    Adding $\sum_{t=1}^T I_{a,t}(\tilde{\ell}_{a,t}+\gamma\sum_{j \in \sA_t}\tilde{\ell}_{j,t})$ to both sides yields
    \begin{nalign}
        \sum_{t=1}^T I_{a,t}\hat{\ell_t} &\leq \frac{\ln{G_T}}{\eta} + \sum_{t=1}^T I_{a,t}\tilde{\ell}_{a,t} + \frac{\eta}{2}\sum_{t=1}^T \sum_{i \in \sA_t} \tilde{\ell}_{i,t} + \gamma\sum_{t=1}^T I_{a,t}\sum_{j \in \sA_t}\tilde{\ell}_{j,t} \\
        &\leq  \frac{\ln{G_T}}{\eta} + \sum_{t=1}^T I_{a,t}\tilde{\ell}_{a,t} + \left(\frac{\eta}{2} + \gamma\right)\sum_{t=1}^T \sum_{i \in \sA_t} \tilde{\ell}_{i,t},
    \end{nalign}
    where the second inequality is due to $I_{a,t} \leq 1$.
\end{proof}

\subsection{Proof of Theorem~\ref{thm:inexpectationregretbound}}
\TheoremInExpectationSBEXP*
\begin{proof}
If an arm $a$ is not in $\sG_T$ then it has never been active in any round, thus $R(a) = 0$. 
For an arm $a \in \sG_T$, taking the expectation of both sides of Lemma~\ref{lemma:boundEstimatedRegretSB-EXP3} and using 
\begin{align*}
  \E_{i_t \sim p_t}[\tilde{\ell}_{i,t}] &= p_{i,t}\frac{\ell_{i,t}}{p_{i,t}} \\
  &= \ell_{i,t} \\
  &\leq 1,
\end{align*}
  we obtain
\begin{align*}
    \E[R(a)] &\leq \frac{\ln{G_T}}{\eta} + \frac{\eta}{2}\sum_{t=1}^TA_t.
\end{align*}
\end{proof}

\subsection{Proof of Theorem~\ref{thm:highprobregretbound}}
Before proving Theorem~\ref{thm:highprobregretbound}, we state the following lemma and its corollary which provide high-probability guarantees that the sum of the loss estimators is a lower confidence bound for the sum of the true losses of all active arms. This lemma is adapted from Lemma 1 of~\cite{Neu2015ExploreNM}.
\begin{lemma}[Lemma 1 of~\cite{Neu2015ExploreNM}]
    For all $i \in [K], t \in [T]$ and $I_{i,t} \in [0,1]$, let $\alpha_{i,t}$ satisfy $0 \leq \alpha_{i,t} \leq 2\gamma I_{i,t}$. With probability $1-\delta'$, 
    \begin{align*}
        \sum_{t=1}^T \sum_{i=1}^K \alpha_{i,t}\I{I_{i,t} > 0}(\tilde{\ell}_{i,t} - \ell_{i,t}) \leq \ln(1/\delta').
    \end{align*}
    \label{lemma:NeusLemmaEstimatorQuality}
\end{lemma}
The proof of Lemma~\ref{lemma:NeusLemmaEstimatorQuality} is identical to that of Lemma 1 of~\cite{Neu2015ExploreNM} (with one additional straightforward step of handling $I_{i,t}=0$) and thus is omitted here. 
For any fixed $j \in \sG_T$, applying Lemma~\ref{lemma:NeusLemmaEstimatorQuality} with $\alpha_{i,t} = 2\gamma I_{i,t}\I{i=j}$ and taking a union bound over all $j \in \sG_T$ leads to the following corollary.
\begin{corollary}
    With probability at least $1-\delta'$,
    \begin{align*}
        \sum_{t=1}^T I_{j,t}(\tilde{\ell}_{j,t} - \ell_{j,t}) \leq \frac{\ln(G_T/\delta')}{2\gamma}
    \end{align*}
    holds simultaneously for all $j \in \sG_T$.
    \label{corollary:NeusCorollary1}
\end{corollary}
\TheoremHighProbSBEXP*
\begin{proof}
Applying Lemma~\ref{lemma:NeusLemmaEstimatorQuality} with $\alpha_{i,t} = (\eta/2 + \gamma)I_{i,t}, \delta' = \delta/2$ and Corollary~\ref{corollary:NeusCorollary1} with $\delta' = \delta/2$ gives
\begin{align}
    \left(\frac{\eta}{2} + \gamma\right)\sum_{t=1}^T \sum_{i \in \sA_t} \tilde{\ell}_{i,t} \leq \ln(2/\delta) + \left(\frac{\eta}{2} + \gamma\right)\sum_{t=1}^T \sum_{i \in \sA_t} \ell_{i,t} 
    \label{eq:concentrateSumtildeell}
\end{align}
and
\begin{align}
    \sum_{t=1}^T I_{a,t} \tilde{\ell}_{a,t} \leq \frac{\ln(2G_T/\delta)}{2\gamma} + \sum_{t=1}^T I_{a,t}\ell_{a,t} \quad\text{for any } a \in [K].
    \label{eq:concentrateItildeell}
\end{align}
Plugging~\eqref{eq:concentrateSumtildeell} and~\eqref{eq:concentrateItildeell} into the right-hand side of Lemma~\ref{lemma:boundEstimatedRegretSB-EXP3} and taking a union bound, we have with probability at least $1-\delta$, simultaneously for all $a \in [K]$,
    \begin{align*}
        \sum_{t=1}^T I_{a,t}\hat{\ell_t} &\leq \frac{\ln{G_T}}{\eta} + \frac{\ln(2G_T/\delta)}{2\gamma} + \sum_{t=1}^T I_{a,t}\ell_{a,t} + \ln(2/\delta)+ \left(\frac{\eta}{2} + \gamma\right)\sum_{t=1}^T\sum_{i=1}^K  I_{i,t}\ell_{i,t}.
    \end{align*}
    Subtracting $\sum_{t=1}^T I_{a,t}\ell_{a,t}$ on both sides and using $I_{i,t}\ell_{i,t} \leq 1$, we obtain
    \begin{align}
        R(a) &\leq \frac{\ln{G_T}}{\eta} + \frac{\ln(2G_T/\delta)}{2\gamma} + \ln(2/\delta) + \left(\frac{\eta}{2} + \gamma\right)\sum_{t=1}^T A_t        
    \end{align}
    with probability at least $1-\delta$.
\end{proof}

\section{PROOFS FOR SECTION~\ref{sec:FTARL}}
Let $D_{\psi_t}: \R^K \times \R^K \to \R_+$ be the Bregman divergence generated by $\psi_t$. Let $\tilde{\ell}_t = \begin{bmatrix}
    \tilde{\ell}_{1,t} \\ \dots \\ \tilde{\ell}_{K,t}
\end{bmatrix}$ be the estimated loss vector in round $t$.
We write $Q_t = \sum_{i=1}^K q_{i,t}$ for the sum of the weights of all arms at the beginning round $t$. Note that $Q_t = 1$. For a set $S \subseteq [K]$, we write $Q_{S, t} = \sum_{i \in S}q_{i,t}$ for the projection of $Q_t$ on $S$. The complement of $S$ is $\bar{S} = [K] \setminus S$. 

\subsection{Proof of Lemma~\ref{lemma:upperboundLQofnonactivearms}}
\LemmaQNonactiveArms*
\begin{proof}
    Since $\tilde{\ell}_{i,t} = 0$ if $i \in \sA_t$ and $i \neq i_t$, the right-hand side can be reduced to 
    \begin{align*}
        \sum_{i \in \sA_t}\tilde{\ell}_{i,t}^2 p_{i,t}^{2-\beta} = \tilde{\ell}_{i_t,t}^2 p_{i_t,t}^{2-\beta}.
    \end{align*}

    For the left-hand side, the estimated losses of inactive arms are equal to
    \begin{align*}
        \hat{\ell}_t - \gamma\sum_{j \in \sA_t}\tilde{\ell}_{j,t} &= \hat{\ell}_t - \gamma \tilde{\ell}_{i_t,t} \\
        &= \hat{\ell}_t - \gamma\frac{\hat{\ell}_t}{p_{i_t,t} + \gamma} \\
        &= \frac{p_{i_t,t}\hat{\ell}_t}{p_{i_t,t} + \gamma}.
    \end{align*}
    Therefore,
    \begin{align*}
        \sum_{i=1}^K \tilde{\ell}_{i,t}^2q_{i,t}^{2-\beta} &= \sum_{i \in \sA_t} \tilde{\ell}_{i,t}^2q_{i,t}^{2-\beta} + \sum_{i \in \bar{\sA}_t} \tilde{\ell}_{i,t}^2q_{i,t}^{2-\beta} \\
        &= \tilde{\ell}_{i_t,t}^2q_{i_t,t}^{2-\beta} + \frac{p_{i_t,t}^2\hat{\ell}_t^2}{(p_{i_t,t} + \gamma)^2}\sum_{i \in \bar{\sA}_t}q_{i,t}^{2-\beta} \\
        &= \frac{\hat{\ell}_t^2}{(p_{i_t,t}+\gamma)^2}q_{i_t,t}^{2-\beta} + \frac{p_{i_t,t}^2\hat{\ell}_t^2}{(p_{i_t,t} + \gamma)^2}\sum_{i \in \bar{\sA}_t}q_{i,t}^{2-\beta} \\
        &= \frac{\hat{\ell}_t^2}{(p_{i_t,t}+\gamma)^2}p_{i_t,t}^{2-\beta}(Q_{\sA_t}^{2-\beta} + p_{i_t,t}^\beta\sum_{i \in \bar{\sA}_t}q_{i,t}^{2-\beta}) \qquad\text{ since } p_{i_t,t} = \frac{q_{i_t,t}}{Q_{\sA_t,t}}\\
        &\leq \frac{\hat{\ell}_t^2}{(p_{i_t,t}+\gamma)^2}p_{i_t,t}^{2-\beta}(Q_{\sA_t}^{2-\beta} + \sum_{i \in \bar{\sA}_t}q_{i,t}^{2-\beta}) \qquad\text{ since } p_{i_t,t}^{\beta} \leq 1\\
        &\leq \tilde{\ell}_{i_t,t}^2p_{i_t,t}^{2-\beta}(Q_{\sA_t} + \sum_{i\in \bar{\sA}_t}q_{i,t}) \\
        &= \sum_{i \in \sA_t}\tilde{\ell}_t^2p_{i,t}^{2-\beta},
    \end{align*}
    where the last two inequalities are due to applying $x^\alpha \leq x$ for all $x \in [0,1], \alpha > 1$ on $p_{i_t,t}, Q_{\sA_t,t}$ and each $q_{i,t}$ for $i \in \bar{\sA}_t$ as well as the fact that $q_t \in \Delta_K$.
\end{proof}

\subsection{Bounding the Estimated Regret}
\begin{lemma}
    For any arm $a \in [K]$, 
    \begin{align*}
        \sum_{t=1}^T I_{a,t}(\hat{\ell}_t - \tilde{\ell}_{a,t}) \leq \frac{K^{1-\beta}}{\eta(1-\beta)}  + \gamma\sum_{t=1}^T \sum_{j \in \sA_t}\tilde{\ell}_{j,t} + \frac{\eta}{2\beta}\sum_{t=1}^T \sum_{i \in \sA_t}\tilde{\ell}_{i,t}p_{i,t}^{1-\beta},
    \end{align*}
    where $e_a$ is the $a$-th standard basis vector of $\R^K$. 
    \label{lemma:FTARLestimatedregretbound}
\end{lemma}
\begin{proof}
    From the regret bound of FTRL using local norms~\cite[e.g.][Lemma 7.12]{OrabonaIntroToOnlineLearningBook} we have 
    \begin{align}
        \sum_{t=1}^T \inp{\tilde{\ell}_t}{q_t - e_a} \leq \psi_{T+1}(e_a) - \min_{x \in \Delta_K}\psi_1(x) + \frac{1}{2}\sum_{t=1}^T \norm{\tilde{\ell}_t}^2_{(\nabla^2 \psi_t(u_t))^{-1}},
        \label{eq:FTRLlocalnorminitial}
    \end{align}
    where $u_t$ is a point between $q_t$ and 
    \begin{align}
        \bar{q}_{t+1} = \argmin_{x \in \R^K} \inp{\tilde{\ell}_t}{x} + D_{\psi_t}(x; q_t).
        \label{eq:barqt1}
    \end{align}
    First, we examine the left-hand side in~\eqref{eq:FTRLlocalnorminitial}. Obviously $\inp{\tilde{\ell}_t}{e_a} = \tilde{\ell}_{a,t}$. Furthermore,
    \begin{nalign}
        \inp{\tilde{\ell}_t}{q_t} &= \sum_{i=1}^K \tilde{\ell}_{i,t}q_{i,t} \\
        &= \sum_{i \in \sA_t} \tilde{\ell}_{i,t}q_{i,t} + \sum_{i \in \bar{\sA}_t} \tilde{\ell}_{i,t}q_{i,t} \\
        &= \hat{\ell}_t \frac{q_{i_t, t}}{p_{i_t,t} + \gamma} + \frac{p_{i_t,t}\hat{\ell}_t}{p_{i_t,t}+\gamma}\sum_{i \in \bar{\sA}_t}q_{i,t} \qquad(\text{since } \tilde{\ell}_{i,t} = \frac{p_{i_t,t}\hat{\ell}_t}{p_{i_t,t} + \gamma}\text{ for } i \in \bar{\sA}_t)\\
        &= \frac{p_{i_t,t}\hat{\ell}_t}{p_{i_t,t}+\gamma}\left(\frac{q_{i_t, t}}{p_{i_t,t}} + \sum_{i \in \bar{\sA}_t}q_{i,t}\right) \\
        &= \frac{p_{i_t,t}\hat{\ell}_t}{p_{i_t,t}+\gamma}\left(\sum_{i \in \sA_t}q_{i,t} + \sum_{i \in \bar{\sA}_t}q_{i,t}\right) \qquad(\text{since } p_{i,t} = \frac{q_{i,t}}{\sum_{j \in \sA_j}q_{j,t}})\\
        &= \frac{p_{i_t,t}\hat{\ell}_t}{p_{i_t,t}+\gamma} \qquad(\text{since } q_t \in \Delta_K)\\
        &= \hat{\ell}_t - \gamma\sum_{j \in \sA_t}\tilde{\ell}_{j,t} .
    \end{nalign}
    Therefore, $\inp{\tilde{\ell}_t}{q_t - e_a} = \hat{\ell}_t - \gamma\sum_{j \in \sA_t}\tilde{\ell}_{j,t} - \tilde{\ell}_{a,t}$. 
    By construction, if $I_{a,t} = 0$ then $\tilde{\ell}_{a,t} = \hat{\ell}_t - \gamma\sum_{j \in \sA_t}\tilde{\ell}_{j,t}$. Hence, 
    \begin{align*}
        \sum_{t=1}^T \inp{\tilde{\ell}_t}{q_t - e_a} = \sum_{t=1}^T I_{a,t} \left(\hat{\ell}_t - \gamma\sum_{j \in \sA_t}\tilde{\ell}_t - \tilde{\ell}_{a,t}\right).
    \end{align*}
    Plugging this into~\eqref{eq:FTRLlocalnorminitial} implies that 
    \begin{align}
        \sum_{t=1}^T I_{a,t}(\hat{\ell}_t - \tilde{\ell}_{a,t}) \leq \psi_{T+1}(e_a) - \min_{x \in \Delta_K}\psi_1(x) + \gamma\sum_{t=1}^T I_{a,t}\sum_{j \in \sA_t}\tilde{\ell}_{j,t}  + \frac{1}{2}\sum_{t=1}^T \norm{\tilde{\ell}_t}^2_{(\nabla^2 \psi_t(u_t))^{-1}}.
        \label{eq:boundestimatedregretut}
    \end{align}
    Next, we bound $\norm{\tilde{\ell}_t}^2_{(\nabla^2 \psi_t(u_t))^{-1}}$ on the right-hand side.
    It can be shown~\citep[e.g.][Section 10.1.2]{OrabonaIntroToOnlineLearningBook} that for the Tsallis entropy regularizer, the solution of the optimization problem~\eqref{eq:barqt1} satisfies $\bar{q}_{i,t+1} \leq q_{i,t}$ whenever $\tilde{\ell}_{i,t} \geq 0$ for all $i \in [K]$. In our construction,
    \begin{itemize}
        \item $\tilde{\ell}_{i,t} = \frac{\ell_{i_t,t}}{p_{i_t,t} + \gamma} \geq 0$ if $i = i_t$,
        \item $\tilde{\ell}_{i,t} = 0$ if $i \in \sA_t$ and $i \neq i_t$,
        \item $\tilde{\ell}_{i,t} = \ell_{i_t,t} - \gamma\sum_{j \in \sA_t}\tilde{\ell}_{j,t} =  \frac{p_{i_t,t}\hat{\ell}_t}{p_{i_t,t} + \gamma} \geq 0$ if $i \notin \sA_t$.
    \end{itemize}
    Hence, the condition $\tilde{\ell}_{i,t} \geq 0$ holds for all $i \in [K]$. It follows that $u_{i,t} \leq q_{i,t}$ for all $i \in [K]$.
    
    It is well-known~\citep[e.g.][]{Abernethy2015Fighting} that the Hessian of Tsallis entropy is diagonal and equal to 
    \begin{align*}
        (\nabla^2\psi_t(x))_{ii} = \frac{\beta}{\eta x_i^{2-\beta}}.
    \end{align*}
    It follows that its inverse is a diagonal matrix with entries $\left(\eta\frac{x_i^{2-\beta}}{\beta}\right)_{i=1,2,\dots,K}$ on the main diagonal. Hence,
    \begin{align*}
        \norm{\tilde{\ell}_t}^2_{(\nabla^2 \psi_t(u_t))^{-1}} &= \frac{\eta}{\beta}\sum_{i=1}^K \tilde{\ell}_{i,t}^2 u_{i,t}^{2-\beta} \\
        &\leq \frac{\eta}{\beta}\sum_{i=1}^K \tilde{\ell}_{i,t}^2 q_{i,t}^{2-\beta} \\
        &\leq \frac{\eta}{\beta}\sum_{i \in \sA_t} \tilde{\ell}_{i,t}^2p_{i,t}^{2-\beta} \\
        &\leq \frac{\eta}{\beta}\sum_{i \in \sA_t} \tilde{\ell}_{i,t}p_{i,t}^{1-\beta},
    \end{align*}
    where the first inequality is due to $u_{i,t} \leq q_{i,t}$, the second inequality is due to Lemma~\ref{lemma:upperboundLQofnonactivearms} and the last inequality is due to $\tilde{\ell}_{i,t}p_{i,t} \leq 1$ for all $i \in \sA_t$. Plugging this into~\eqref{eq:boundestimatedregretut} and using $I_{a,t} \leq 1$ gives 
    \begin{align*}
        \sum_{t=1}^T I_{a,t}(\hat{\ell}_t - \tilde{\ell}_{a,t}) &\leq \psi_{T+1}(e_a) - \min_{x \in \Delta_K}\psi_1(x) + \gamma\sum_{t=1}^T \sum_{j \in \sA_t}\tilde{\ell}_{j,t} + \frac{\eta}{2\beta}\sum_{t=1}^T \sum_{i \in \sA_t}\tilde{\ell}_{i,t}p_{i,t}^{1-\beta} \\
        &\leq \frac{K^{1-\beta}}{\eta(1-\beta)} + \gamma\sum_{t=1}^T \sum_{j \in \sA_t}\tilde{\ell}_{j,t} + \frac{\eta}{2\beta}\sum_{t=1}^T \sum_{i \in \sA_t}\tilde{\ell}_{i,t}p_{i,t}^{1-\beta},
    \end{align*}
    where, in the last inequality, where we used $\psi_{T+1}(e_a) - \min_{x \in \Delta_K}\psi_1(x) \leq \frac{K^{1-\beta}}{\eta(1-\beta)}$ as a standard property of the Tsallis entropy regularizer~\citep[e.g.][]{Abernethy2015Fighting}.
\end{proof}

\subsection{Proof of Theorem~\ref{thm:FTARLinexpectationbound}}
\TheoremInExpectationFTARL*
\begin{proof}
With $\gamma = 0$, Lemma~\ref{lemma:FTARLestimatedregretbound} implies that 
\begin{align*}
    \sum_{t=1}^T I_{a,t}(\hat{\ell}_t - \tilde{\ell}_{a,t}) &\leq \frac{K^{1-\beta}}{\eta(1-\beta)} + \frac{\eta}{2\beta}\sum_{t=1}^T \sum_{i \in \sA_t}\tilde{\ell}_{i,t}p_{i,t}^{1-\beta} \\
    &= \frac{K^{1-\beta}}{\eta(1-\beta)} + \frac{\eta}{2\beta}\sum_{t=1}^T \ell_{i_t,t}p_{i_t,t}^{-\beta},
\end{align*}
where the equality is due to the fact that for $i \in \sA_t$, $\ell_{i,t} = 0$ if $i \neq i_t$ and $\tilde{\ell}_{i_t,t} = \frac{\ell_{i_t,t}}{p_{i_t,t}}$.
Taking the expectation over $i_t \sim p_t$ on both sides and using 
\begin{align*}
     \E_{i_t \sim p_t}[\ell_{i_t,t}p_{i_t,t}^{-\beta}] &= \sum_{i \in \sA_t} p_{i,t}^{1-\beta}\ell_{i,t} \\
    &\leq \sum_{i \in \sA_t} p_{i,t}^{1-\beta} \\
    &\leq A_t^{\beta},
\end{align*}
where the first inequality is due to $\ell_{i,t} \in [0,1]$ and the second inequality is Holder's inequality, we obtain 
\begin{align*}
    \E[R(a)] &\leq \frac{K^{1-\beta}}{\eta(1-\beta)} + \frac{\eta}{2\beta}TA^\beta.
\end{align*}
\end{proof}

\subsection{Proof of Theorem~\ref{thm:FTARLhighprobbound}}
\TheoremHighProbFTARL*
\begin{proof}
We apply Lemma~\ref{lemma:NeusLemmaEstimatorQuality} twice: 
\begin{itemize}
    \item the first time with $\delta' = \delta/3, \alpha_{i,t}=2\gamma I_{i,t} p_{i,t}^{1-\beta}$ to obtain 
    \begin{align}
        \sum_{t=1}^T \sum_{i \in \sA_t}p_{i,t}^{1-\beta}\tilde{\ell}_{i,t} \leq \frac{\ln(3/\delta)}{2\gamma} + \sum_{t=1}^T \sum_{i \in \sA_t}p_{i,t}^{1-\beta}\ell_{i,t}
        \label{eq:concentrate1}
    \end{align}
    with probability at least $1-\delta/3$,
    \item the second time with $\delta' = \delta/3, \alpha_{i,t}=2\gamma I_{i,t}$ to obtain 
    \begin{align}
        \sum_{t=1}^T \sum_{i \in \sA_t}\tilde{\ell}_{i,t} \leq \frac{\ln(3/\delta)}{2\gamma} + \sum_{t=1}^T\sum_{i \in \sA_t}\ell_{i,t}
        \label{eq:concentrate2}
    \end{align}    
    with probability at least $1-\delta/3$.
\end{itemize}
We also apply Corollary~\ref{corollary:NeusCorollary1} once with $\delta' = \frac{\delta}{3}$ to obtain 
\begin{align}
    \sum_{t=1}^T I_{a,t}\tilde{\ell}_{a,t} \leq \frac{\ln(3G_T/\delta)}{2\gamma} + \sum_{t=1}^T I_{a,t}\ell_{a,t}
    \label{eq:concentrate3}
\end{align}
with probability at least $1-\delta/3$. 
Plugging~\eqref{eq:concentrate1},~\eqref{eq:concentrate2},~\eqref{eq:concentrate3} into Lemma~\ref{lemma:FTARLestimatedregretbound} and taking a union bound, we obtain with probability at least $1-\delta$,
\begin{align*}
    \sum_{t=1}^T I_{a,t}\hat{\ell}_t &\leq \frac{K^{1-\beta}}{\eta(1-\beta)} + \frac{\ln(3G_T/\delta)}{2\gamma} + \sum_{t=1}^T I_{a,t}\ell_{a,t} + \frac{\ln(3/\delta)}{2} + \gamma \sum_{t=1}^T\sum_{i \in \sA_t}\ell_{i,t} + \frac{\eta\ln(3/\delta)}{4\beta\gamma} + \frac{\eta}{2\beta}\sum_{t=1}^T \sum_{i \in \sA_t} p_{i,t}^{1-\beta}\ell_{i,t}.
\end{align*}
Subtracting $\sum_{t=1}^T I_{a,t}\ell_{a,t}$ on both sides, we obtain
\begin{align*}
    R(a) &\leq \frac{K^{1-\beta}}{\eta(1-\beta)} + \frac{\ln(3G_T/\delta)}{2\gamma} + \frac{\ln(3/\delta)}{2} + \gamma \sum_{t=1}^T\sum_{i \in \sA_t}\ell_{i,t} + \frac{\eta\ln(3/\delta)}{4\beta\gamma} + \frac{\eta}{2\beta}\sum_{t=1}^T \sum_{i \in \sA_t} p_{i,t}^{1-\beta}\ell_{i,t}\\
    &\leq \frac{K^{1-\beta}}{\eta(1-\beta)} + \frac{\ln(3G_T/\delta)}{2\gamma} + \left(\frac{\eta}{4\beta\gamma} + \frac{1}{2}\right)\ln(3/\delta) + \gamma AT + \frac{\eta}{2\beta}\sum_{t=1}^T \sum_{i \in \sA_t} p_{i,t}^{1-\beta} \\
    &\leq \frac{K^{1-\beta}}{\eta(1-\beta)} + \frac{\eta}{2\beta}TA^{\beta} + \gamma AT + \left(\frac{\eta}{4\beta\gamma} + \frac{1}{2}\right)\ln(3/\delta) + \frac{\ln(3G_T/\delta)}{2\gamma} \\
    &\leq \Qchanged{\frac{K^{1-\beta}}{\eta(1-\beta)} + \frac{\eta}{2\beta}TA^{\beta} + \gamma AT + \left(\frac{\eta}{4\beta\gamma} + \frac{1}{2}\right)\ln(3/\delta) + \frac{\ln(3K/\delta)}{2\gamma}},
\end{align*}
where 
\begin{itemize}
    \item the second inequality is due to $\ell_{i,t} \in [0,1]$ and $A_t \leq A$,
    \item the third inequality is Holder's inequality and $A_t \leq A$,
    \item \Qchanged{the last inequality is due to $G_T \leq K$.}
\end{itemize}
\end{proof}

\section{PROOFS FOR SECTION~\ref{sec:RealValuedCases}}
\label{appendix:confidenceregret}
\begin{algorithm2e}[t]
	\KwIn{$\eta > 0, \gamma > 0, \eta \leq \gamma$}
	Initialize $\tilde{q}_{i, 1} = 1$ for $i = 1, 2, \dots, K$\; 
	\For{each round $t = 1, \dots, T$}{
        An adversary selects and reveals $K$ values $I_{i,t} \in [0,1]$\;	
        Compute $w_{i,t} = I_{i,t}\tilde{q}_{i,t}$ and $W_t = \sum_{i=1}^K w_{i,t}$\;
        Compute $p_{i,t} = \frac{w_{i,t}}{W_t}$ \;
		Draw $i_t \sim p_t$ and observe $\hat{\ell}_t = \ell_{i_t, t}$\;
		Construct loss estimate $\tilde{\ell}_{i, t}$ by Equation~\eqref{eq:lossestimatorIXReal} \;
		Update $\tilde{q}_{i,t+1}$ by Equation~\eqref{eq:qtildeupdateruleReal}.
	}
	\caption{SB-EXP3 for experts that report their confidences with bandit feedback}
	\label{algo:SBEXP3IXReal}
\end{algorithm2e}

First, we state a more general definition of the active sets of arms. An arm $i$ is active in round $t$ if $I_{i,t} > 0$ i.e. $\sA_t = \{i \in [K]: I_{i,t} > 0\}$. Let $\sG_t = \cup_{s=1}^t \sA_t$ and $G_t = \abs{\sG_t}$. Let the potential function $\tilde{Q}_t$ be
\begin{align*}
    \tilde{Q}_{t} = \sum_{i \in \sG_T} \tilde{q}_{i,t}.
\end{align*}
The protocol is given in Algorithm~\ref{algo:SBEXP3IXReal}.
In round $t$, the loss estimator is 
\begin{align}
    \tilde{\ell}_{i,t} = \begin{cases}
        \frac{\ell_{i,t}\I{i_t = i}}{p_{i,t} + \gamma I_{i,t}} \quad\text{if } p_{i,t} > 0,\\
        0 \quad\text{otherwise}.
    \end{cases}
    \label{eq:lossestimatorIXReal}
\end{align}
and the update rule for $\tilde{q}_{i,t+1}$ becomes
\begin{align}
    \tilde{q}_{i, t+1} = \tilde{q}_{i, t}\exp\left(\eta I_{i,t}\left(\hat{\ell}_t - \tilde{\ell}_{i,t} - \gamma\sum_{j=1}^K I_{j,t}\tilde{\ell}_{j,t}\right)\right).
    \label{eq:qtildeupdateruleReal}
\end{align}

Before proving Theorem~\ref{thm:highprobregretboundReal}, we present the following lemma which bounds the growth of $\frac{\tilde{Q}_{t+1}}{\tilde{Q}_t}$.
\begin{lemma}
    For any $t \geq 1$, 
    \begin{align*}
        \frac{\tilde{Q}_{t+1}}{\tilde{Q}_t} &\leq 1 + \eta \frac{W_t}{\tilde{Q}_t}\sum_{i=1}^K p_{i,t}\left((\hat{\ell}_t - \tilde{\ell}_{i,t} - \gamma\sum_{j=1}^K I_{j,t}\tilde{\ell}_{j,t}) + \eta I_{i,t}(\hat{\ell}_t - \tilde{\ell}_{i,t} - \gamma\sum_{j=1}^K I_{j,t}\tilde{\ell}_{j,t})^2\right) 
    \end{align*}
    \label{lemma:upperboundQt1QtfirstReal}
\end{lemma}
\begin{proof}
    By the definition of $\tilde{Q}_{t+1}$, we have 
    \begin{align*}
        \tilde{Q}_{t+1} - \tilde{Q}_t &= \sum_{i \in \sG_T} \tilde{q}_{i,t}\exp\left(\eta I_{i,t}\left(\hat{\ell}_t - \tilde{\ell}_{i,t}- \gamma\sum_{j=1}^K I_{j,t}\tilde{\ell}_{j,t}\right)\right) - \sum_{i \in \sG_T} \tilde{q}_{i,t}        \\
        &= \sum_{i \in \sG_T} \tilde{q}_{i,t}\left(\exp\left(\eta I_{i,t}(\hat{\ell}_t - \tilde{\ell}_{i,t} - \gamma\sum_{j=1}^K I_{j,t}\tilde{\ell}_{j,t})\right)-1\right) \\
        &\leq \sum_{i \in \sG_T} \tilde{q}_{i,t}\left( \eta I_{i,t}(\hat{\ell}_t - \tilde{\ell}_{i,t}- \gamma\sum_{j=1}^K I_{j,t}\tilde{\ell}_{j,t}) +  \eta^2 I_{i,t}^2(\hat{\ell}_t - \tilde{\ell}_{i,t}- \gamma\sum_{j=1}^K I_{j,t} \tilde{\ell}_{j,t})^2\right) \\
        &= \eta \sum_{i \in \sG_T} I_{i,t}\tilde{q}_{i,t}\left( (\hat{\ell}_t - \tilde{\ell}_{i,t}- \gamma\sum_{j=1}^K I_{j,t} \tilde{\ell}_{j,t}) +  \eta I_{i,t}(\hat{\ell}_t - \tilde{\ell}_{i,t}- \gamma\sum_{j=1}^K I_{j,t} \tilde{\ell}_{j,t})^2\right) \\
        &= \eta W_t \sum_{i \in \sG_T} p_{i,t}\left( (\hat{\ell}_t - \tilde{\ell}_{i,t}- \gamma\sum_{j=1}^K I_{j,t} \tilde{\ell}_{j,t}) +  \eta I_{i,t}(\hat{\ell}_t - \tilde{\ell}_{i,t}- \gamma\sum_{j=1}^K I_{j,t} \tilde{\ell}_{j,t})^2\right),
    \end{align*}
    where 
    \begin{itemize}
        \item the inequality is obtained by applying $\exp(x) - 1 \leq x + x^2$ for any $x \leq 1$ on $x = \eta I_{i,t}(\hat{\ell}_t - \tilde{\ell}_{i,t}- \gamma\sum_{j=1}^K I_{j,t}\tilde{\ell}_{j,t})$ and multiplying both sides by $\tilde{q}_{i,t} \geq 0$,
        \item the last equality is due to the computation of $p_{i,t} = \frac{I_{i,t}\tilde{q}_{i,t}}{W_t}$.
    \end{itemize}    
    Dividing both sides by $\tilde{Q}_{t} > 0$, we obtain the desired expression.
\end{proof}
Note that in each round $t$, Lemma~\ref{lemma:IXcomputeEipt} still holds for $I_{i,t} \in [0,1]$. This implies the following corollary.
\begin{corollary}
    For any $t \geq 1$, 
    \begin{align*}
        \frac{\tilde{Q}_{t+1}}{\tilde{Q}_t} \leq 1 + \eta^2 \sum_{i=1}^K p_{i,t}I_{i,t}\left(\hat{\ell}_t - \tilde{\ell}_{i,t} - \gamma\sum_{j=1}^K I_{j,t}\tilde{\ell}_{j,t}\right)^2.
    \end{align*}    
    \label{corollary:upperboundQt1QtsecondReal}
\end{corollary}
\begin{proof}
    We write the second term in the sum on the right-hand side of Lemma~\ref{lemma:upperboundQt1QtfirstReal} as follows: 
    \begin{align*}
        \frac{\tilde{Q}_{t+1}}{\tilde{Q}_t} &\leq 1 + \eta \frac{W_t}{\tilde{Q}_t}\sum_{i=1}^K p_{i,t}\left((\hat{\ell}_t - \tilde{\ell}_{i,t}- \gamma\sum_{j=1}^K I_{j,t} \tilde{\ell}_{j,t}) + \eta I_{i,t}(\hat{\ell}_t - \tilde{\ell}_{i,t}- \gamma\sum_{j=1}^K I_{j,t} \tilde{\ell}_{j,t})^2\right) \\
        &= 1 + \underbrace{\frac{\eta W_t}{\tilde{Q}_t}\sum_{i=1}^K p_{i,t}(\hat{\ell}_t - \tilde{\ell}_{i,t}- \gamma\sum_{j=1}^K I_{j,t} \tilde{\ell}_{j,t})}_{(a)} + \underbrace{\frac{\eta^2 W_t}{\tilde{Q}_t} \sum_{i=1}^K p_{i,t}I_{i,t}(\hat{\ell}_t - \tilde{\ell}_{i,t}- \gamma\sum_{j=1}^K I_{j,t} \tilde{\ell}_{j,t})^2}_{(b)} \\
    \end{align*}
We bound $(a)$ and $(b)$ separately. By Lemma~\ref{lemma:IXcomputeEipt} we have 
\begin{align*}
    \sum_{i=1}^K p_{i,t}(\hat{\ell}_t - \tilde{\ell}_{i,t}- \gamma\sum_{j=1}^K I_{j,t} \tilde{\ell}_{j,t}) &= \hat{\ell}_t - \sum_{i=1}^K p_{i,t}\tilde{\ell}_{i,t}- \gamma\sum_{j=1}^K I_{j,t} \tilde{\ell}_{j,t} \\
    &= 0,
\end{align*}
and thus the quantity $(a)$ is equal to $0$. To bound $(b)$, we  have
\begin{align*}
    W_t &= \sum_{i=1}^K I_{i,t}\tilde{q}_{i,t} \\
    &\leq \sum_{i=1}^K \tilde{q}_{i,t} \\
    &= \tilde{Q}_t,
\end{align*}
where the inequality is due to $I_{i,t} \in [0,1]$ and $\tilde{q}_{i,t} > 0$. This implies that $0 < \frac{W_t}{\tilde{Q}_t} \leq 1$. Multiplying both sides by $\eta^2\sum_{i=1}^K p_{i,t}I_{i,t}(\hat{\ell}_t - \tilde{\ell}_{i,t}- \gamma\sum_{j=1}^K \tilde{\ell}_{j,t})^2 \geq 0$, we obtain
\begin{align*}
    (b) \leq \eta^2\sum_{i=1}^K p_{i,t}I_{i,t}\left(\hat{\ell}_t - \tilde{\ell}_{i,t}- \gamma\sum_{j=1}^K I_{j,t} \tilde{\ell}_{j,t}\right)^2,
\end{align*}
which implies the desired statement.
\end{proof}

\subsection{Proof of Theorem~\ref{thm:highprobregretboundReal}}
\TheoremHighProbRealCase*
\begin{proof}
    We still employ the standard strategy of lower and upper bounding $Q_{T+1}$. We have 
    \begin{nalign}
        \ln{\tilde{Q}_{T+1}} \geq \ln{\tilde{q}_{a, T+1}} = \eta\sum_{t=1}^T I_{a,t}\left(\hat{\ell}_t - \tilde{\ell}_{a,t} - \gamma\sum_{j=1}^K I_{j,t} \tilde{\ell}_{j,t}\right).
        \label{eq:lowerBoundQT1Real}
    \end{nalign}
    On the other hand, 
    \begin{nalign}
        \ln{\tilde{Q}_{T+1}} &= \ln{\tilde{Q}_1} + \sum_{t=1}^T \ln{\frac{\tilde{Q}_{t+1}}{\tilde{Q}_t}} \\
        &\leq \ln{G_T} + \sum_{t=1}^T \ln\left(1 + \eta^2 \sum_{i=1}^K p_{i,t}I_{i,t}(\hat{\ell}_t - \tilde{\ell}_{i,t}- \gamma\sum_{j=1}^K I_{j,t} \tilde{\ell}_{j,t})^2\right) \\
        &\leq \ln{G_T} + \eta^2\sum_{t=1}^T \sum_{i=1}^K p_{i,t}I_{i,t}\left(\hat{\ell}_t - \tilde{\ell}_{i,t}- \gamma\sum_{j=1}^K I_{j,t} \tilde{\ell}_{j,t}\right)^2
        \label{eq:upperBoundQT1Reala}
    \end{nalign}
    where the first inequality is due to Corollary~\ref{corollary:upperboundQt1QtsecondReal} and the second inequality is due to $\ln(1+x) \leq x$ for all $x \geq -1$.
    Let $c_t = \gamma\sum_{j=1}^K I_{j,t} \tilde{\ell}_{j,t}$. For any $t \geq 1$, we have: 
    \begin{align*}
        \hat{\ell}_t - c_t &= \hat{\ell}_t - \gamma I_{i_t,t} \tilde{\ell}_{i_t,t} \qquad(\text{ since }\tilde{\ell}_{j,t} = 0\text{ if } j \neq i_t)\\
        &= \hat{\ell}_t - \frac{\gamma I_{i_t,t}\hat{\ell}_t}{p_{i_t,t} + \gamma I_{i_t,t}}\\
        &= \frac{p_{i_t,t}\hat{\ell}_t}{p_{i_t,t}+\gamma I_{i_t,t}} \\
        &= p_{i_t,t}\tilde{\ell}_{i_t,t}.
    \end{align*}
    It follows that
    \begin{nalign}
        \sum_{i=1}^K p_{i,t}I_{i,t}(\hat{\ell}_t - \tilde{\ell}_{i,t}- c_t)^2 &= \sum_{i=1}^K p_{i,t}I_{i,t}(\hat{\ell}_t - c_t)^2 + \sum_{i=1}^K p_{i,t}I_{i,t}\tilde{\ell}_{i,t}^2 - 2\sum_{i=1}^K p_{i,t}I_{i,t} (\hat{\ell}_t - c_t)\tilde{\ell}_{i,t} \\
        &\leq (\hat{\ell}_t - c_t)^2\sum_{i=1}^K p_{i,t}I_{i,t} + \sum_{i=1}^K p_{i,t}I_{i,t}\tilde{\ell}_{i,t}^2 \\
        &\leq \sum_{i=1}^K p_{i,t}I_{i,t} + \sum_{i=1}^K I_{i,t}\tilde{\ell}_{i,t},
        \label{eq:upperBoundQT1Realb}
    \end{nalign}
    where 
    \begin{itemize}
        \item the first inequality is due to $\hat{\ell}_t - c_t = p_{i_t,t}\tilde{\ell}_{i_t,t} \geq 0$,
        \item the second inequality is due to $\hat{\ell}_t - c_t = p_{i_t,t}\tilde{\ell}_{i_t,t} \leq 1$.
    \end{itemize}
    
    Combining~\eqref{eq:lowerBoundQT1Real},~\eqref{eq:upperBoundQT1Reala} and~\eqref{eq:upperBoundQT1Realb}, we obtain 
    \begin{align}
        \sum_{t=1}^T I_{a,t}(\hat{\ell}_t - \tilde{\ell}_{a,t} - \gamma\sum_{j=1}^K I_{j,t} \tilde{\ell}_{j,t}) &\leq \frac{\ln{G_T}}{\eta} + \eta \sum_{t=1}^T \left( \sum_{i=1}^K p_{i,t}I_{i,t} + \sum_{i=1}^K I_{i,t}\tilde{\ell}_{i,t} \right),
        \label{eq:highprobregretboundReala}
    \end{align}
    which implies that 
    \begin{align}
        \sum_{t=1}^T I_{a,t}\hat{\ell}_t &\leq \sum_{t=1}^T I_{a,t}\tilde{\ell}_{a,t}  + \frac{\ln{G_T}}{\eta} + \eta\sum_{t=1}^T \sum_{i=1}^K p_{i,t}I_{i,t}  + \sum_{t=1}^T\sum_{i=1}^K \left(I_{a,t}\gamma + \eta \right)I_{i,t}\tilde{\ell}_{i,t} \\
        &\leq \sum_{t=1}^T I_{a,t}\tilde{\ell}_{a,t}  + \frac{\ln{G_T}}{\gamma} + \gamma\sum_{t=1}^T \sum_{i=1}^K p_{i,t}I_{i,t}  + 2\gamma\sum_{t=1}^T\sum_{i=1}^K I_{i,t}\tilde{\ell}_{i,t},
    \end{align}
    where the last inequality is due to $I_{a,t} \leq 1$ and picking $\eta = \gamma$.
    
     Note that Lemma~\ref{lemma:NeusLemmaEstimatorQuality} holds for $I_{i,t} \in [0,1]$. Applying Lemma~\ref{lemma:NeusLemmaEstimatorQuality} with $\alpha_{i,t} = 2\gamma I_{i,t}, \delta' = \delta/2$, applying Corollary~\ref{corollary:NeusCorollary1} with $\delta' = \delta/2$ and using $\sum_{i=1}^K p_{i,t}I_{i,t} \leq \sum_{i=1}^K I_{i,t}$, we obtain that with probability at least $1-\delta$, simultaneously for all $a \in [K]$,
    \begin{align*}
        \sum_{t=1}^T I_{a,t}\hat{\ell}_t &\leq \sum_{t=1}^T I_{a,t}\ell_{a,t} + \frac{\ln(2G_T/\delta)}{2\gamma} + \frac{\ln{G_T}}{\gamma} + \gamma\sum_{t=1}^T \sum_{i=1}^K I_{i,t} + \ln(2/\delta) + 2\gamma\sum_{t=1}^T\sum_{j=1}^K I_{j,t} \ell_{j,t}.
    \end{align*}    
    Subtracting $\sum_{t=1}^T I_{a,t}\ell_{a,t}$ on both sides and using $\ell_{j,t} \leq 1$ and $\gamma, \delta \in (0,1)$, we obtain 
    \begin{align}
        R(a) &\leq \frac{\ln(2G_T/\delta)}{\gamma} + \frac{\ln{G_T}}{\gamma} + \ln(2/\delta) + 3\gamma\sum_{t=1}^T \sum_{j=1}^K I_{j,t} \\
        &\leq \frac{3\ln(G_T/\delta)}{\gamma} + 3\gamma\sum_{t=1}^T\sum_{i=1}^K I_{i,t}.
    \end{align}
    Setting $\gamma = \eta = \sqrtfrac{\ln(G_T/\delta)}{\sum_{t=1}^T\sum_{i=1}^K I_{i,t}}$ leads to the desired bound.
\end{proof}

\section{PROOFS FOR SECTION~\ref{sec:EXP4}}
\label{appendix:SEEXP4}
Recall that the sleeping experts are considered \textit{sleeping augmented arms}. Note that $\sB_t = \{m \in [M]: I_{m,t} = 1\}$ is the set of awake experts as defined in the main text.
For an expert $u \in [M]$, the actual loss of expert $u$ in round $t$ is
\begin{align*}
    x_{u,t}  = \inp{E_{u,t}}{\ell_t}.
\end{align*} 
First, we prove a technical lemma showing that the $z_t$-weighted average of these estimated losses of the augmented arms is equivalent to $\hat{\ell}_t - \gamma\sum_{j=1}^K \tilde{\ell}_{j,t}$. 
This lemma is the counterpart of Lemma~\ref{lemma:IXcomputeEipt}.
\begin{lemma}
    For any $t \in [K]$ and $m \in \sB_t$,
    \begin{align*}
        \E_{m \sim z_t}[\tilde{x}_{m,t}] = \hat{\ell}_t - \gamma\sum_{j=1}^K \tilde{\ell}_{j,t}.
    \end{align*}
    \label{lemma:expectationofXmt}
\end{lemma}
\begin{proof}
    Let $E_{m,t}^{(k)}$ be the value of the element at index $k$ in $E_{m,t}$. We have 
    \begin{align*}
        \E_{m \sim z_t}[\tilde{x}_{m,t}] &= \sum_{m \in \sB_t} z_{m,t} \tilde{x}_{m,t} \\
        &= \sum_{m \in \sB_t}z_{m,t}\sum_{k=1}^K E_{m,t}^{(k)}\tilde{\ell}_{k,t} \\
        &= \sum_{m \in \sB_t}z_{m,t} E_{m,t}^{(i_t)}\tilde{\ell}_{i_t,t} \\
        &= \frac{\hat{\ell}_t}{p_{i_t,t} + \gamma}\sum_{m \in \sB_t}z_{m,t} E_{m,t}^{(i_t)} \\
        &= \frac{\hat{\ell}_t p_{i_t,t}}{p_{i_t,t}+\gamma} \\
        &= \hat{\ell}_t - \gamma \frac{\hat{\ell}_t}{p_{i_t,t} + \gamma} \\
        &= \hat{\ell}_t - \gamma\tilde{\ell}_{i_t,t},
    \end{align*}
    where 
    \begin{itemize}
        \item the second equality is by Equation~\eqref{eq:xmt}
        \item the third equality is due to $\tilde{\ell}_{k,t} = 0$ whenever $k \neq i_t$
        \item the fourth equality is due to $p_{k,t} = \sum_{m \in \sB_t}z_{m,t}E_{m,t}^{(k)}$ for all $k \in [K]$.
    \end{itemize}    
\end{proof}
Let $\tilde{Q}_t = \sum_{m = 1}^M \tilde{q}_{m,t}$. For a set $S \in [M]$, let $\tilde{Q}_{S,t} = \sum_{m \in S}\tilde{q}_{m,t}$ be the projection of $\tilde{Q}_t$ on $S$. Let $\bar{S} = [M] \setminus S$ for any $S \subseteq [M]$. Lemma~\ref{lemma:expectationofXmt} leads to the following technical lemma that resembles Lemma~\ref{lemma:upperboundQt1Qt}.
\begin{lemma}
    For any $t \geq 0$,
    \begin{align*}
        \frac{\tilde{Q}_{t+1}}{\tilde{Q}_t} \leq \sum_{m=1}^M z_{m,t}\exp\left(\eta(\hat{\ell}_t - \gamma\sum_{j=1}^K \tilde{\ell}_{j,t}- \tilde{x}_{m,t})\right).
    \end{align*}
    \label{lemma:upperboundQt1QtEXP4}
\end{lemma}
\begin{proof}
    The proof makes use of the following two facts that can be proved easily:
    \begin{itemize}
        \item Fact 1: The function $f(x) = e^{-\eta x}$ is convex for any $\eta \in \R$.
         \item Fact 2: For any $a, b > 0, c \geq 0$, if $a \geq b$ then 
        \begin{align*}
            \frac{a}{b} \geq \frac{a+c}{b+c}.
        \end{align*}
    \end{itemize}
    By Jensen's inequality and Fact 1, we have 
    \begin{align*}
        \sum_{m=1}^M z_{m,t}\exp(-\eta\tilde{x}_{m,t}) &= \E_{m \sim z_t}[\exp(-\eta\tilde{x}_{m,t})]  \\
        &\geq \exp\left(-\eta\E_{m \sim z_t}[\tilde{x}_{m,t}]\right) \\
        &= \exp\left(-\eta(\hat{\ell}_t - \gamma\sum_{j=1}^K \tilde{\ell}_{j,t})\right),
    \end{align*}
    where the last equality is due to Lemma~\ref{lemma:expectationofXmt}. Since $z_{m,t} = 0$ for $m \notin \sB_t$, the expression above is equivalent to 
    \begin{align*}
        \sum_{m \in \sB_t} z_{m,t}\exp(-\eta\tilde{x}_{m,t}) \geq \exp\left(-\eta(\hat{\ell}_t - \gamma\sum_{j=1}^K \tilde{\ell}_{j,t})\right).
    \end{align*}
    Multiplying $\exp(\eta(\hat{\ell}_t - \gamma\sum_{j=1}^K \tilde{\ell}_{j,t}))\tilde{Q}_{\sB_t, t} > 0$ on both sides, we obtain 
    \begin{align}
        \sum_{m \in \sB_t} z_{m,t}\tilde{Q}_{\sB_t, t}\exp\left(\eta(\hat{\ell}_t- \gamma\sum_{j=1}^K \tilde{\ell}_{j,t}-\tilde{x}_{m,t})\right) \geq \tilde{Q}_{\sB_t, t}.
        \label{eq:pitQAtfirstEXP4}
    \end{align}
    By definition, $z_{m,t} = \frac{\tilde{q}_{m,t}}{\tilde{Q}_{\sB_t, t}}$. Hence, Equation~\eqref{eq:pitQAtfirstEXP4} is equivalent to 
    \begin{align*}
        \sum_{m \in \sB_t} \tilde{q}_{m,t}\exp\left(\eta(\hat{\ell}_t-\gamma\sum_{j=1}^K \tilde{\ell}_{j,t}-\tilde{x}_{m,t})\right) \geq \tilde{Q}_{\sB_t, t}.
        \label{eq:pitQAtsecondEXP4}
    \end{align*}
    By our update rule, $\tilde{q}_{m,t+1} = \tilde{q}_{m,t}$ for $m \notin \sB_t$ and $\tilde{q}_{m,t+1} = \tilde{q}_{m,t}\exp\left(\eta(\hat{\ell}_t - \gamma\sum_{j=1}^K \tilde{\ell}_{j,t} - \tilde{x}_{m,t})\right)$ for $m \in \sB_t$. Hence,
    \begin{align*}
        \sum_{m \in \sB_t} \tilde{q}_{m,t+1} = \sum_{m \in \sB_t} \tilde{q}_{m,t}\exp\left(\eta(\hat{\ell}_t-\gamma\sum_{j=1}^K \tilde{\ell}_{j,t}-\tilde{x}_{m,t})\right) \geq \tilde{Q}_{\sB_t, t}.
    \end{align*}
    Applying Fact 2 for $a = \sum_{m \in \sB_t} \tilde{q}_{m,t+1}, b = \tilde{Q}_{\sB_t, t}$ and $c = \tilde{Q}_{\bar{\sB}_t, t}$, we obtain 
    \begin{nalign}
        \sum_{i=m}^M z_{m,t}\exp\left(\eta(\hat{\ell}_t - \gamma\sum_{j=1}^K \tilde{\ell}_{j,t}-\tilde{x}_{m,t})\right) &= \sum_{m \in \sB_t} z_{m,t}\exp\left(\eta(\hat{\ell}_t - \gamma\sum_{j=1}^K \tilde{\ell}_{j,t}-\tilde{x}_{m,t})\right)\\
        &= \frac{\sum_{m \in \sB_t} \tilde{q}_{m,t}\exp\left(\eta(\hat{\ell}_t-\gamma\sum_{j=1}^K \tilde{\ell}_{j,t}-\tilde{x}_{m,t})\right)}{\tilde{Q}_{\sB_t, t}} \\
        &= \frac{\sum_{m \in \sB_t} \tilde{q}_{m,t+1}}{\tilde{Q}_{\sB_t, t}} \\
         &\geq \frac{\sum_{m \in \sB_t} \tilde{q}_{m,t+1} + \tilde{Q}_{\bar{\sB}_t, t}}{\tilde{Q}_{\sB_t, t} + \tilde{Q}_{\bar{\sB}_t, t}} \\
        &= \frac{\tilde{Q}_{t+1}}{\tilde{Q}_t},
    \end{nalign}
    where the last equality is due to the fact that $\sum_{m \in \bar{\sB}_t}\tilde{q}_{m,t+1} = \sum_{m \in \bar{\sB}_t}\tilde{q}_{m,t} = \tilde{Q}_{\bar{\sB}_t, t}$.
\end{proof}

\subsection{Bounding the Estimated Regret}
The following lemma bounds the estimated regret of each expert $u \in [M]$.

\begin{lemma}
    For any $\gamma \geq 0$, for any $u \in [M]$, SE-EXP4 guarantees that
    \begin{align}
        \sum_{t=1}^T I_{u,t}(\hat{\ell}_t - \tilde{x}_{u,t}) &\leq \frac{\ln{M}}{\eta} + \left(\gamma + \frac{\eta}{2}\right)\sum_{t=1}^T \sum_{j=1}^K  \tilde{\ell}_{j,t}.
    \end{align}
    \label{lemma:generalBoundEstimatedRegretSEEXP4}
\end{lemma}
\begin{proof}
    We have 
\begin{nalign}
    \ln{\tilde{Q}_{T+1}} &= \ln{\sum_{m=1}^K \tilde{q}_{m,T+1}} \\
    &\geq \ln{\tilde{q}_{u, T+1}} \\
    &= \eta\sum_{t=1}^T I_{u,t}(\hat{\ell}_t - \gamma\sum_{j=1}^K \tilde{\ell}_{j,t} -\tilde{x}_{u,t}).
    \label{eq:lowerBoundQT1EXP4}
\end{nalign}
On the other hand, we have
\begin{align*}
    \ln{\tilde{Q}_{T+1}} &= \ln{\tilde{Q}_1} + \sum_{t=1}^T \ln{\frac{\tilde{Q}_{t+1}}{\tilde{Q}_t}} \\
    &\leq \ln{M} + \sum_{t=1}^T \ln\left(\sum_{m \in \sB_t} z_{m,t}\exp\left(\eta(\hat{\ell}_t - \gamma\sum_{j=1}^K \tilde{\ell}_{j,t}- \tilde{x}_{m,t})\right)\right) \\
    &= \ln{M} + \sum_{t=1}^T \ln\left(\exp(\eta(\hat{\ell}_t-\gamma\sum_{j=1}^K \tilde{\ell}_{j,t}))\sum_{m \in \sB_t}z_{m,t}\exp(-\eta\tilde{x}_{m,t})\right) \\
    &= \ln{M} + \sum_{t=1}^T \left( \eta(\hat{\ell}_t-\gamma\sum_{j=1}^K \tilde{\ell}_{j,t}) + \ln\left(\sum_{m \in \sB_t}z_{m,t}\exp(-\eta\tilde{x}_{m,t})\right) \right) \\
    &\leq \ln{M} + \sum_{t=1}^T \left( \eta(\hat{\ell}_t - \gamma\sum_{j=1}^K \tilde{\ell}_{j,t}) + \ln\left(\sum_{m \in \sB_t}z_{m,t}(1 + \frac{\eta^2\tilde{x}_{m,t}^2}{2} - \eta\tilde{x}_{m,t})\right) \right)  \\
    &= \ln{M} + \sum_{t=1}^T \left( \eta(\hat{\ell}_t-\gamma\sum_{j=1}^K \tilde{\ell}_{j,t}) + \ln\left(1 + \eta^2\sum_{m \in \sB_t} \frac{z_{m,t}\tilde{x}_{m,t}^2}{2} - \eta \sum_{m \in \sB_t}z_{m,t}\tilde{x}_{m,t}\right) \right) \\
    &\leq \ln{M} + \sum_{t=1}^T \left( \eta(\hat{\ell}_t-\gamma\sum_{j=1}^K \tilde{\ell}_{j,t}) + \eta^2\sum_{m \in \sB_t} \frac{z_{m,t}\tilde{x}_{m,t}^2}{2} - \eta \sum_{m \in \sB_t}z_{m,t}\tilde{x}_{m,t} \right) \\
    &= \ln{M} + \eta^2\sum_{t=1}^T \sum_{m \in \sB_t} \frac{z_{m,t}\tilde{x}_{m,t}^2}{2},
\end{align*}
where 
\begin{itemize}
    \item the first inequality is due to Lemma~\ref{lemma:upperboundQt1QtEXP4},
    \item the second inequality is $\exp(-x) \leq 1 + \frac{x^2}{2}-x$ for all $x \geq 0$,
    \item the third inequality is $\ln(1+x) \leq x$ for all $x \geq -1$,
    \item the last equality is due to Lemma~\ref{lemma:expectationofXmt}.
\end{itemize}
We obtain 
\begin{align}
    \sum_{t=1}^T I_{u,t}(\hat{\ell}_t - \gamma\sum_{j=1}^K \tilde{\ell}_{j,t} - \tilde{x}_{u,t}) \leq \frac{\ln{M}}{\eta} + \frac{\eta}{2}\sum_{t=1}^T \sum_{m \in \sB_t} z_{m,t}\tilde{x}_{m,t}^2.
\end{align}
Next, we proceed in the same way as in~\cite{Neu2015ExploreNM}. We have
\begin{align}
    \sum_{m \in \sB_t}z_{m,t}\tilde{x}_{m,t}^2 &= \sum_{m \in \sB_t}z_{m,t} \left(\sum_{k=1}^K E_{m,t}^{(k)}\tilde{\ell}_{k,t}\right)^2 \\
    &\leq \sum_{m \in \sB_t}z_{m,t} \sum_{k=1}^K E_{m,t}^{(k)}(\tilde{\ell}_{k,t})^2 \\
    &= \sum_{k=1}^K (\tilde{\ell}_{k,t})^2\sum_{m \in \sB_t}z_{m,t}E_{m,t}^{(k)} \\
    &= \sum_{k=1}^K p_{k,t}(\tilde{\ell}_{k,t})^2 \\
    &\leq \sum_{k=1}^K \tilde{\ell}_{k,t},
\end{align}
where the first inequality is Jensen's inequality. This implies 
\begin{align}
    \sum_{t=1}^T I_{u,t}(\hat{\ell}_t - \gamma\sum_{j=1}^K \tilde{\ell}_{j,t} - \tilde{x}_{u,t}) &\leq \frac{\ln{M}}{\eta} + \frac{\eta}{2}\sum_{t=1}^T\sum_{k=1}^K \tilde{\ell}_{k,t}.
\end{align}
Moving $\gamma\sum_{t=1}^T I_{u,t}\sum_{j=1}^K \tilde{\ell}_{j,t}$ to the right-hand side and using $I_{u,t} \in [0,1]$, we obtain the desired statement.
\end{proof}

\subsection{Proof of Theorem~\ref{thm:highprobregretboundEXP4}}
\TheoremHighProbSEEXP*
\begin{proof}
    Lemma~\ref{lemma:generalBoundEstimatedRegretSEEXP4} implies that
    \begin{nalign}
        \sum_{t=1}^T I_{u,t}\hat{\ell}_t &\leq \frac{\ln{M}}{\eta} + \sum_{t=1}^T I_{u,t} \tilde{x}_{u,t} + \sum_{t=1}^T \sum_{j=1}^K \left(\gamma + \frac{\eta}{2}\right) \tilde{\ell}_{j,t} \\
        &= \frac{\ln{M}}{\eta} + \sum_{t=1}^T I_{u,t}\sum_{j=1}^K E_{u,t}^{(j)} \tilde{\ell}_{j,t} + \sum_{t=1}^T \sum_{j=1}^K \left(\gamma + \frac{\eta}{2}\right) \tilde{\ell}_{j,t}.
        \label{eq:SEEXP4EstimatedRegret}
    \end{nalign}
    Next, we apply Lemma~\ref{lemma:NeusLemmaEstimatorQuality} twice, where 
    \begin{itemize}
        \item the first time with $\alpha_{i,t} = 2\gamma I_{u,t}E_{u,t}^{(i)}, \delta' = \frac{\delta}{2M}$ and a union bound over $[M]$ implies
        \begin{align}
            \sum_{t=1}^T\sum_{j=1}^K I_{u,t}E_{u,t}^{(j)} (\tilde{\ell}_{j,t} - \ell_{j,t}) \leq \frac{\ln(2M/\delta)}{2\gamma}
            \label{eq:concentrateSEEXP4-1}
        \end{align}
        with probability at least $1-\delta/2$;
        \item the second time with $\alpha_{i,t} = \gamma + \eta/2, \delta' = \delta/2$ implies
        \begin{align}
            \left(\gamma + \frac{\eta}{2}\right)\sum_{t=1}^T \sum_{j=1}^K (\tilde{\ell}_{j,t} - \ell_{j,t}) \leq \ln(2/\delta)
            \label{eq:concentrateSEEXP4-2}
        \end{align}
        with probability at least $1-\delta/2$.        
    \end{itemize}    
    Plugging~\eqref{eq:concentrateSEEXP4-1} and ~\eqref{eq:concentrateSEEXP4-2} into~\eqref{eq:SEEXP4EstimatedRegret} yields
    \begin{align*}
        \sum_{t=1}^T I_{u,t}\hat{\ell}_t &\leq \frac{\ln{M}}{\eta} + \sum_{t=1}^T I_{u,t}\sum_{j=1}^K E_{u,t}^{(j)}\ell_{j,t} + \frac{\ln(2M/\delta)}{2\gamma} + \left(\gamma + \frac{\eta}{2}\right)\sum_{t=1}^T\sum_{j=1}^K \ell_{j,t} + \ln(2/\delta) \\
        &\leq \frac{\ln{M}}{\eta} + \sum_{t=1}^T I_{u,t}x_{u,t} + \frac{\ln(2M/\delta)}{2\gamma} + \left(\gamma + \frac{\eta}{2}\right)TK + \ln(2/\delta). 
    \end{align*}
    Moving $\sum_{t=1}^T I_{u,t}x_{u,t}$ to the left-hand side, we obtain 
    \begin{align}
        R(u) \leq \frac{\ln{M}}{\eta} + \frac{\ln(2M/\delta)}{2\gamma} + \left(\gamma + \frac{\eta}{2}\right)TK + \ln(2/\delta).
    \end{align}
Letting $\eta = 2\gamma$ and tuning $\eta$ implies the $O(\sqrt{TK\ln(M/\delta)})$ bound.
\end{proof}

\subsection{A Pseudo-Regret Bound of SE-EXP4}
We bound the pseudo-regret $\E[R(u)]$ for any expert $u \in [M]$. Note that $\gamma = 0$.
Taking the expectation on both side of Lemma~\ref{lemma:generalBoundEstimatedRegretSEEXP4} and using 
\begin{align*}
    \E_{i_t \sim p_t}\left[\tilde{\ell}_{j,t}\right] &= \ell_{j,t} \leq 1,
\end{align*}
we obtain
\begin{align}
    \E[\sum_{t=1}^T I_{u,t}(\hat{\ell}_t - \tilde{x}_{u,t})] &\leq \frac{\ln{M}}{\eta} + \frac{\eta}{2}TK.
\end{align}
On the other hand,
\begin{align*}
    \E_{i_t \sim p_t}[\tilde{x}_{u,t}] &= \sum_{k=1}^K p_{k,t}\E[\tilde{x}_{u,t} \mid i_t = k] \\
    &= \sum_{k=1}^K p_{k,t} \E[\sum_{j=1}^K E_{u,t}^{(j)}\tilde{\ell}_{j,t} \mid i_t = k] \\
    &= \sum_{k=1}^K p_{k,t}E_{u,t}^{(k)}\frac{\ell_{k,t}}{p_{k,t}} \\
    &= \sum_{k=1}^K E_{u,t}^{(k)}\ell_{k,t} \\
    &= \inp{E_{u,t}}{\ell_t}.
\end{align*}
We conclude that 
\begin{align}
    \E[R(u)] \leq \frac{\ln{M}}{\eta} + \frac{\eta TK}{2}.
\end{align}
By setting $\eta = \sqrt{\frac{2\ln{M}}{TK}}$ we obtain the bound 
\begin{align}
\E[R(u)] \leq \sqrt{2TK\ln{M}}.
\end{align}

\subsection{Proof of Theorem~\ref{thm:adaptiveBound}}
\TheoremAdaptiveBound*
\begin{proof}
For any triple $(k, t_1, t_2)$ we create a virtual expert that is active from round $t_1$ to round $t_2$ and give advice $e_k$. There are $M = K{T \choose 2} = \frac{KT(T+1)}{2}$ such experts. 
        Because $M \leq KT^2 \leq (KT)^2$, we have 
        \begin{align*}
            \ln{M} \leq \ln((KT)^2) = 2\ln(KT).
        \end{align*}
        Furthermore, for all $\delta \in (0,1)$,
        \begin{align*}
            \ln(2M/\delta) \leq \ln(4M/\delta^2) \leq \ln((2KT/\delta)^2) =  2\ln(2KT/\delta).
        \end{align*}
        Let $u_{k,t_1,t_2}$ denote the expert indexed by $(k,t_1,t_2)$. By Theorem~\ref{thm:highprobregretboundEXP4}, with probability at least $1-\delta$,
    \begin{nalign}
        R_{[t_1, t_2]}(k) &= R(u_{k,t_1,t_2}) \\
        &\leq \frac{\ln{M}}{\eta} + \frac{\ln(2M/\delta)}{2\gamma} + (\gamma + \frac{\eta}{2})TK + \ln(2/\delta) \\
        &\leq \frac{2\ln(KT)}{\eta} + \frac{\ln(2KT/\delta)}{\gamma} + (\gamma + \frac{\eta}{2})TK + \ln(2/\delta)
    \end{nalign}   
    holds simultaneously for all $u_{k, t_1, t_2}$.
\end{proof}

\subsection{Proof of Corollary~\ref{corollary:sqrtStrackingbound}}
\CorollaryTrackingBound*
\begin{proof}
Without loss of generality, assume that $T$ is divisible by $S$. We use the following algorithm called SE-EXP4-Restart, which runs in $S$ episodes. In each episode, a new instance of SE-EXP4 with virtual experts is run for $T/S$ rounds. Let $b = 1, 2, \dots, S$ be an index for the episodes, and $t_{(b)} = \frac{bT}{S}$ be the ending round of episode $b$. Note that each episode $b$ starts from round $\frac{(b-1)T}{S} + 1$ to round $\frac{bT}{S}$.

    We examine the regret of this SE-EXP4-Restart with respect to the competing arms $j_{1:T}$ in each episode $b$. Divide $T$ rounds into $Z = H(j_{1:T}) + 1$ non-overlapping segments, where the competing arms are the same within each segment $z = 1, \dots, Z$. Let $S_z, E_z$ be the first and last rounds of segment $z$. For every pair $(b, z)$ of episode $b$ and segment $z$, let
    \begin{align}
        F_{b, z} = [S_z, E_z] \cap [\frac{(b-1)T}{S}, \frac{bT}{S}]
    \end{align}
    be the intersection between the rounds of episode $b$ and segment $z$.
    Since the episodes are non-overlapping and the segments are non-overlapping, the intervals $F_{b, z}$'s are non-overlapping. In addition, their union is $\cup_{b,z}F_{b,z} = [T]$.

    Fix an episode $b$ and a segment $z$. There are two cases:
    \begin{itemize}
        \item $F_{b,z} = \emptyset$: Obviously, the regret of SE-EXP4-Restart on this empty interval is zero.
        \item $F_{b,z} \neq \emptyset$: in this case, because $F_{b,z}$ is an interval within episode $b$, the tracking regret of SE-EXP4-Restart on $F_{b,z}$ cannot exceed the adaptive regret of running (a new instance of) SE-EXP4 with virtual experts under horizon $T/S$. By Theorem~\ref{thm:adaptiveBound}, this is bounded by 
        \begin{align}
            R_{F_{b,s}} &\leq \frac{2}{\eta}\ln\left(\frac{KT}{S\delta}\right) + \eta TK/S + \ln(2S/\delta)
        \end{align}
        with probability at least $1-\delta/S$.       
    \end{itemize}
    Taking a union bound over all $S$ episodes and setting $\eta = 2\gamma$ implies that with probability at least $1-\delta$,
    \begin{align*}
        R_{F_{b,s}} &\leq \frac{2}{\eta}\ln\left(\frac{KT}{S\delta}\right) + \eta TK/S + \ln(2S/\delta)
    \end{align*}
    simultaneously for all intervals $F_{b,s}$.
    Because $F_{b,z}$'s are non-overlapping and their union is $[T]$, the tracking regret of SE-EXP4-Restart is bounded by 
    \begin{align}
        R(j_{1:T}) &= \sum_{b,z}R_{F_{b,z}} \\
        &= \left(\sum_{b,z} \I{F_{b,z} \neq \emptyset}\right)\left(\frac{2}{\eta}\ln\left(\frac{KT}{S\delta}\right) + \eta TK/S + \ln(2S/\delta)\right).
    \end{align}
    Next, we show that the count $C = \sum_{b,z}\I{F_{b,z} \neq \emptyset}$ is smaller than $2S$. 
    Observe that the $S$ episodes split the sequence of $T$ rounds into $S$ intervals with $S-1$ splitting points (not counting the two ends at $0$ and $T$).
    Similarly, the $S+1$ segments of $j_{1:T}$ have $S$ splitting points.
    In total, there are at most $2S-1$ splitting points from the episodes and the segments of $j_{1:T}$.
    Each non-empty interval $F_{b,z}$ has an ending point that is either $T$ or one of the $2S-1$ splitting points. Therefore, there can be at most $2S$ such $F_{b,z}$. We conclude that $C \leq 2S$. As a result, with probability at least $1-\delta$,
    \begin{align*}
        R(j_{1:T}) \leq \frac{4S}{\eta}\ln\left(\frac{KT}{S\delta}\right) + 2\eta TK + 2S\ln(2S/\delta).
    \end{align*}
    Letting $\eta = \sqrtfrac{2S\ln\left(\frac{KT}{S\delta}\right)}{TK}$ leads to the desired bound.
    \end{proof}

\section{PROOFS FOR SECTION~\ref{sec:AdaptiveLowerBound}}
\TheoremLowerBound*
\begin{proof}
Let $\gA$ be any algorithm with the stated worst-case guarantee and $f(T, A) = O(T^{\gamma}A^\beta (\ln(T))^\mu)$ represent the worst-case regret bound of $\gA$. 
    We show a construction with $A_t = 2$ for all $t = 1, \dots, T$.     
    Our construction is adapted from the lower bound construction of~\cite{Daniely2015StronglyAdaptiveOL} for strongly adaptive regret in the standard adversarial MAB setting. Without loss of generality, assume $4f(T, A)$ divides $T$. Let $L = \frac{T}{4f(T, A)}$ and $K = 1 + 4f(T, A)$. Obviously $L \geq \Omega(T^{1-\gamma}A^{-\beta}(\ln(T))^{-\mu})$. Consider an environment $V_0$ defined as follows: 
    \begin{itemize}
        \item Arm $1$ is always active. Its losses are $\ell_{1, t} = 0.5$ for all $t \in [T]$.
        \item Arm $k = 2, 3, \dots, K$ are active only for the rounds in the interval $\gI_k = \left[\frac{(k-2)T}{4f(T, A)} + 1, \frac{(k-1)T}{4f(T,A)}\right]$, respectively. The length of each interval is $L$. Their losses are $\ell_{k,t} = 1$ for the rounds $t$ in which they are active.
    \end{itemize}
    Figure~\ref{fig:V0} illustrates this environment $V_0$.     
    \begin{figure}[t]
        \centering
        \begin{tikzpicture}
        \draw (0,0.2) -- + (0,-0.4) node[below] {1};
        \draw (12.5,0.2) -- + (0,-0.4) node[below] {$T$};
        \draw (0,0) -- node[above=2mm, pos=0.5] {arm 1, $\ell_{1,t}=0.5$} + (12.5,0);
        \draw (0,-1.2) -- + (0,-0.4) node[below] {1};
        \draw (2.5,-1.2) -- + (0,-0.4) node[below] {};
        \draw (0,-1.4) -- node[above=2mm, pos=0.5] {arm 2} node[below=2mm, pos=0.5]{$\ell_{2,t}=1$} + (2.5,0);
        \draw (5,-1.2) -- + (0,-0.4) node[below] {};
        \draw (2.5,-1.4) -- node[above=2mm, pos=0.5] {arm 3} node[below=2mm, pos=0.5]{$\ell_{3,t}=1$} + (2.5,0);
        \draw[dotted] (5, -1.4) -- + (5, 0);
        \draw (10, -1.2) -- + (0,-0.4) node[below]{};
        \draw (12.5,-1.2) -- + (0,-0.4) node[below] {$T$};
        \draw (10,-1.4) -- node[above=2mm, pos=0.5] {arm $K$} node[below=2mm, pos=0.5]{$\ell_{K,t}=1$} + (2.5,0);
        \end{tikzpicture}
        \caption{Environment $V_0$. All arms have loss equal 1 when they are active}
        \label{fig:V0}
\end{figure}
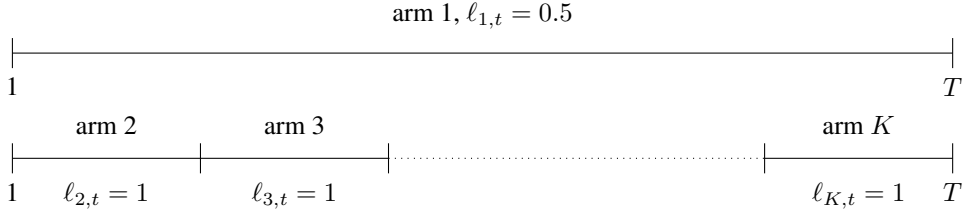
We also define $K-1$ competing environments $V_{k}$ for $k = 2, 3, \dots, K$, defined as follows: 
\begin{itemize}
    \item The number of arms and their active rounds are identical to that of $V_0$. That is, arm 1 is always active for all rounds while each of the arms $k = 2, 3, \dots, K$ are active for rounds within $\gI_k$, respectively.
    \item The losses of every arm are the same as in $V_0$, except for that of arm $k$: all losses of arm $k$ are $0$. 
    Figure~\ref{fig:Vkj} illustrates environment $V_{k,j}$.
\end{itemize}
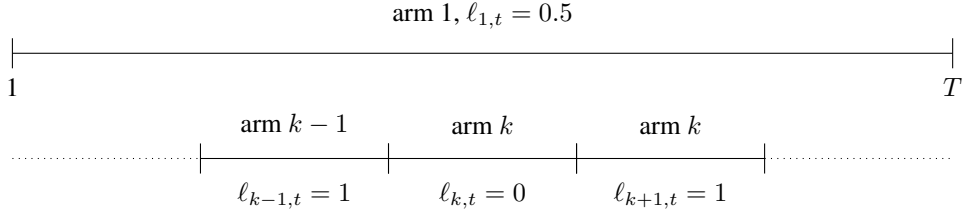
\begin{figure}[t]
    \centering
    \begin{tikzpicture}
    [dot/.style = {circle,inner sep=0pt, minimum size=2pt}]
    \draw (0,0.2) -- + (0,-0.4) node[below] {1};
    \draw (12.5,0.2) -- + (0,-0.4) node[below] {$T$};
    \draw (0,0) -- node[above=2mm, pos=0.5] {arm 1, $\ell_{1,t}=0.5$} + (12.5,0);
    \draw[dotted] (0, -1.4) -- + (2.5, 0);
    \draw (2.5,-1.2) -- + (0,-0.4) node[below] {};
    \draw (5,-1.2) -- + (0,-0.4) node[below] {};
    \draw (2.5,-1.4) -- node[above=2mm, pos=0.5] {arm $k-1$}node[below=2mm, pos=0.5]{$\ell_{k-1,t}=1$} + (2.5,0);
    \draw (5,-1.2) -- + (0,-0.4) node[below] {};
    \draw (7.5,-1.2) -- + (0,-0.4) node[below] {};
    \draw (5,-1.4) -- node[above=2mm, pos=0.5] {arm $k$} node[below=2mm, pos=0.5]{$\ell_{k,t}=0$}  + (2.5,0);
    \draw (7.5,-1.2) -- + (0,-0.4) node[below] {};
    \draw (10,-1.2) -- + (0,-0.4) node[below] {};
    \draw (7.5,-1.4) -- node[above=2mm, pos=0.5] {arm $k$} node[below=2mm, pos=0.5]{$\ell_{k+1,t}=1$}  + (2.5,0);
    \draw[dotted] (10, -1.4) -- + (2.5, 0);
    \end{tikzpicture}
    \caption{Environment $V_{k}$. Except for arm $k$ which has losses equal to $0$, all arms have losses equal to $1$ when they are active.}
    \label{fig:Vkj}
\end{figure}
Following the standard strategy of comparing the behavior of the algorithm on $V_0$ and competing environments~\citep{EXP3Auer2002b, Daniely2015StronglyAdaptiveOL}, we first consider the neutral environment $V_0$. 
Let $\E_0$ and $\Pr_0$ indicate the expectation and probability taken in this environment over the randomness of the algorithm $\gA$, respectively. Let $U = \{t: i_t \neq 1\}$ be the rounds in which the arm chosen by $\gA$ is not arm $1$. On $V_0$, since arm $1$ is the best arm and the gaps between the losses of arm $1$ and that of every other arm is $0.5$, the inequality $\sup_{a}\E_0[R(a)] \leq f(T, A)$ implies 
\begin{align}
    \E_0\left[\abs{U}\right] \leq 2f(T, A).
    \label{eq:boundE0U}
\end{align}
For any $k \in \{2, 3, \dots, K\}$, let 
\begin{align*}
    \gE_k = \{U \cap \gI_k = \emptyset\}
\end{align*}
be the event that only arm $1$ is chosen by $\gA$ on $\gI_k$. 
Because the $\gI_k$ are non-overlapping and $\cup_{k=2,\dots,K}\gI_k = [T]$, we can write 
\begin{align*}
    U = \cup_{k=2, \dots,K} (U \cap \gI_k),
\end{align*}
and 
\begin{align*}
    \abs{U}= \sum_{k=2}^K \abs{U \cap \gI_k}.
\end{align*}
Next, we show that for some segment $\gI_{k^*}$ of size $L = \frac{T}{4f(T,A)}$, we have $\E_0\left[\abs{U \cap \gI_{k^*}}\right] \leq \frac{1}{2}$. Assume on the contrary that $\E_0\left[\abs{U \cap \gI_{k}}\right] > \frac{1}{2}$ for all $k = 2, \dots, K$. Then,
\begin{align*}
    \E_0[U] &= \sum_{k=2}^K \E_0\left[U \cap \gI_k\right] \\
    &> \frac{K-1}{2} \\
    &= 2f(T, A),
\end{align*}
which contradicts~\eqref{eq:boundE0U}. Since $\abs{U \cap \gI_{k^*}}$ is a non-negative integer, the inequality $\E_0\left[\abs{U \cap \gI_{k^*}}\right] \leq \frac{1}{2}$ implies $\Pr_0\left[\gE_{k^*}\right] \geq \frac{1}{2}$.
Next, we consider the environment $V_{k^*}$. From round $1$ up to (and including) round $t^* = \frac{(k^*-2)T}{4f(T, A)}$, the set of active arms and their losses on $V_0$ and $V_{k^*}$ are identical. As a result, the distribution over the past chosen arms and observed losses induced by $\gA$ up to round $t^*$ is the same on both environment. Moreover, once the algorithm enters $\gI_{k^*}$ at round $t^*+1$ and chooses only arm $1$ subsequently, it also observes the same sequence of chosen arms and losses on both $V_0$ and $V_{k^*}$ until the end of $\gI_{k^*}$. Hence,
\begin{align}
    {\Pr}_{k^*}\left[\gE_{k^*}\right] = {\Pr}_0\left[\gE_{k^*}\right]\geq \frac{1}{2},
\end{align}
where the subscript $k^*$ indicates a probability measured in environment $V_{k^*}$.
In other words, on $V_{k^*}$, with probability at least $0.5$, arm 1 is always chosen on $I_{k^*}$ and arm $k^*$ is never chosen. Under this event $\gE_{k^*}$, the regret of $\gA$ with respect to arm $k^*$ is $(0.5 - 0)L = 0.5L$. When $\gE_{k^*}$ does not hold, the regret with respect to $k^*$ is non-negative because $\ell_{k^*,t} = 0$. Overall, on $V_{k^*}$ the expected regret is at least
\begin{align*}
    \E_{k^*}\left[R(k^*)\right] \geq \frac{L}{4}.
\end{align*}
\end{proof}

\section{PROOF OF THEOREM~\ref{thm:anytimeinexpectationbound}: DOUBLING-TRICK FOR ADAPTING TO $\sum_{t=1}^T A_t$ AND $G_T$}
\label{appendix:anytime}
For SB-EXP3, computing the optimal learning rates require the fraction $\sqrtfrac{\ln{G_T}}{\sum_{t=1}^T A_t}$ of $\sqrt{\ln{G_T}}$ and the sum $\sqrt{\sum_{t=1}^T A_t}$. Both of these quantities are monotonically non-decreasing. Therefore, we can apply the doubling trick on these two quantities. First, we prove a simple lemma justifying doing this.
\begin{lemma}
    Let $a, b, c, d > 0$ be constants such that $a \leq c$ and $b \leq d$. Let 
    \begin{align*}
        f(x) = \frac{a}{x} + \frac{bx}{2}
    \end{align*}
    be a function on $\R_+$. Then,
    \begin{align*}
        f\left(\sqrt\frac{2c}{d}\right) \leq \sqrt{2cd}.
    \end{align*}
    \label{lemma:justifydoublingtrick}
\end{lemma}
\begin{proof}
    Due to $a \leq c$ and $b \leq d$, for any $x \geq 0$ we have 
    \begin{align*}
        f(x) \leq \frac{c}{x} + \frac{dx}{2}.
    \end{align*}
    Plugging $x = \sqrtfrac{2c}{d}$ into the right-hand side gives
    \begin{align*}
        f\left(\sqrt\frac{2c}{d}\right) &\leq c\sqrtfrac{d}{2c} + d\sqrtfrac{c}{2d} \\
        &= \sqrt{2cd}.
    \end{align*}
\end{proof}
Lemma~\ref{lemma:justifydoublingtrick} implies that for any horizon $T$, if we set $\eta_t = \sqrt{\frac{2c}{d}}$ for $c$ and $c$ such that $c \geq \ln{G_T}$ and $d \geq \sum_{t=1}^T A_t$ then we obtain a regret of bound $\sqrt{2cd}$. 

We proceed to perform the doubling trick on $\ln{G_T}$ and $\sum_{t=1}^T A_t$. The full procedure is given in Algorithm~\ref{algo:AnyTimeSBEXP3}. The main idea is a two-level doubling trick which divides the learning process into episodes as follows: 
\begin{itemize}
    
\item The first level: throughout the learning process, we maintain a set $\sV$ for the arms that have been active at least once in each episode and an upper bound $2^C$ for $\ln(\abs{\sV})$. Initially, $\sV = \emptyset$ and $C = 1$. 
At the beginning of round $t$, we check if $\ln(\abs{\sV \cup \sA_t})$ exceeds $2^C$. If $\ln(\abs{\sV \cup \sA_t}) \leq 2^C$ then we continue the learning process and update $\sV = \sV \cup \sA_t$. Otherwise, we reset $\sV$ to $\emptyset$, increment $C$ by at least one until $2^C \geq \ln(A_t)$ and start a new episode from round $t$.

\item The second level: throughout the rounds of each episode, we maintain a cumulative sum $U$ for the sum of $A_t$ and an upper bound $2^b$ for $U$. Note that $C$ is fixed within these rounds. Before the first round of an episode, we initialize $U = 0, b = 1$. As long as $U + A_t \leq 2^b$, we run SB-EXP3 with $\eta = \sqrtfrac{2^{C+1}}{2^b}$ and update $U = U + A_t$. Once $U$ exceeds $2^b$ at some round $t$, we increment $b$ by at least one until $A_t \leq 2^b$, reset $U = 0$ and run a new instance of SB-EXP3 onwards.

\end{itemize}

\begin{algorithm2e}[t]
    \SetAlgoNoEnd
    Initialize $U = 0, C = 1, b = 1, \sV = \emptyset$\;
    Initialize $L_{i} = 0$ for $i = 1, 2, \dots, K$\;
	\For{\upshape {each round} $t = 1, \dots, $}{
        An adversary selects and reveals $\sA_t$\;	
        \If{$\ln(\abs{\sV \cup \sA_t}) > 2^C$}
        {
            $C = C + 1$\;
            \While{$\ln(A_t) > 2^C$}
            {
                $C = C + 1$\;
            }
            Set $\sV = \emptyset$\;
            Set $U = 0, b = 1$\;
            \For{\upshape arm $i \in \sG_t$}
            {
                $L_i = 0$\;
            }            
        }
        \If{$U + A_t > 2^b$}{            
            $b = b + 1$\;
            \While{$A_t > 2^b$}{
                $b = b + 1$\;
            }
            Set U = 0\;
            \For{\upshape arm $i \in \sG_t$}
            {
                $L_i = 0$\;
            }            
        }    
        Update $\sV = \sV \cup \sA_t$\;
        Update $U = U + A_t$\;
        Compute $\eta = \sqrtfrac{2^{C+1}}{2^b}$\;    
        \For{\upshape arm $i \in \sA_t$}
        {
            $\tilde{q}_{i,t} = \exp(\eta L_{i})$
        }
        Compute $p_t$ by~\eqref{eq:pit}\;
        Sample $i_t \sim p_t$\;
        Compute $\tilde{\ell}_{i,t}$ by~\eqref{eq:lossestimator}\;
        \For{\upshape arm $i \in \sA_t$}
        {
            $L_i = L_i + \hat{\ell}_t - \tilde{\ell}_{i,t} - \gamma\sum_{j \in \sA_t}\tilde{\ell}_{j,t}$\;
        }
	}
	\caption{SB-EXP3-ATGT adapted to $G_T$ and $\sum_{t=1}^T A_t$}
	\label{algo:AnyTimeSBEXP3}
\end{algorithm2e}
The pseudo-regret of Algorithm~\ref{algo:AnyTimeSBEXP3} is shown in the following theorem.
\TheoremAnytimeInExpectation*
\begin{proof} 
    Let $C_T$ be the last value of $C$ after $T$ rounds. Note that $C_T$ is also the number of episodes. Let $c = 1, 2, \dots, C_T$ be the index of the episodes. We first bound the regret within each episode $c$, and then sum up this bound over $C_T$ episodes to get the total regret bound.

\subsection*{Bounding The Regret Within Each Episode}
Fix an episode $c$.
Let $\sT_c$ be the rounds in this episode, and $T_c = \abs{\sT_c}$.
Let $\sV_c$ be the set of arms that are active at least once during this episode, and $V_c = \abs{\sV_c}$.
By construction, $b = 1$ at the beginning of this episode. 
Let $B$ be the last value of $b$ after $T_c$ rounds starting from the first round of episode $c$.
Divide the rounds in $\sT_c$ into $B$ time intervals, where $b$ does not change in each interval. For $b = 1, 2, \dots, B$, let $F_b$ be the time interval of $b$. Let $U_b = \sum_{t \in F_b}A_t$ be the sum of $A_t$ within $F_b$. Since $U_b \leq 2^b$ and $\ln(V_c) \leq 2^c$, by Lemma~\ref{lemma:justifydoublingtrick} and Theorem~\ref{thm:inexpectationregretbound}, the regret of the learner in this interval is bounded by
\begin{align*}
    \max_{a \in [K]}\E[R_{F_b}(a)] \leq \sqrt{2^b2^{c+1}}.
\end{align*}
Let $R_c(a)$ be the regret incurred during episode $c$ with respect to arm $a$. We have 
\begin{align*}
    \max_{a \in [K]}\E[R_c(a)] &\leq \sum_{b=1}^{B}\max_{a \in [K]}\E[R_{F_b}(a)] \\
    &\leq \sum_{b=1}^{B} \sqrt{2^b2^{c+1}} \\
    &\leq \frac{\sqrt{2}}{\sqrt{2}-1} \sqrt{2^{B}2^{c+1}}.
\end{align*}
For $b = 1, \dots, B$, let $t_b$ be the first round of $F_b$. Since $b$ is increased from $B-1$ to $B$ at the beginning of round $t_{B}$, we have 
$U_{B-1} + A_{t_{B}} > 2^{B - 1}$. It follows that $2^{B} \leq 2(U_{B-1} + A_{t_{B}}) \leq 2\sum_{t \in \sT_c}A_t$. Hence, 
\begin{align}
    \max_{a \in [K]}\E[R_c(a)] \leq \frac{\sqrt{2}}{\sqrt{2}-1}\sqrt{2^{c+1}\sum_{t \in \sT_c}A_t}.
    \label{eq:eachepisode}
\end{align}
\subsection*{Bounding The Total Regret}
Summing up~\eqref{eq:eachepisode} for $c = 1, \dots, C_T$, we obtain 
\begin{align*}
    \max_{a \in [K]}\E[R(a)] &\leq \sum_{c = 1}^{C_T} \max_{a \in [K]}\E[R_c(a)] \\
    &\leq \frac{\sqrt{2}}{\sqrt{2}-1}\sum_{c=1}^{C_T} \sqrt{2^{c+1} \sum_{t \in \sT_c}A_t} \\
    &\leq \frac{\sqrt{2}}{\sqrt{2}-1} \left(\sqrt{\sum_{t=1}^T A_t}\right)\sum_{c=1}^{C_T} \sqrt{2^{c+1}} \\
    &\leq \frac{2\sqrt{2}}{(\sqrt{2}-1)^2} \left(\sqrt{\sum_{t=1}^T A_t}\right) \sqrt{2^{C_T}},
\end{align*}
where the third inequality is due to $\sum_{t \in \sT_c}A_t \leq \sum_{t=1}^T A_t$ and the last inequality is due to $\sum_{c=1}^{C_T}\sqrt{2^c} = \frac{\sqrt{2}}{\sqrt{2}-1}(\sqrt{2^{C_T}}-1)$.

Assume $C$ was increased at least once i.e. $C_T > 1$, otherwise we immediately have an $O\left(\sqrt{\sum_{t=1}^T A_t}\right)$ total regret bound. 
Let $\tau$ be the first round of the last episode.
Since $C$ was increased from $C_T - 1$ to $C_T$ at round $\tau$, we have $2^{C_T-1} < \ln(\abs{\sV_{C_T-1} \cup \sA_{\tau}})$. On the other hand, since $(\sV_{C_T-1} \cup \sA_{\tau}) \subseteq \sG_T$, we have $\ln(\abs{\sV_{C_T-1} \cup \sA_{\tau}}) \leq \ln(G_T)$.  This implies that $2^{C_T} \leq 2\ln(G_T)$, hence the total regret is bounded by 
\begin{align*}
    \max_{a \in [K]}\E[R(a)] \leq \frac{4}{(\sqrt{2}-1)^2}\sqrt{\ln(G_T)\sum_{t=1}^T A_t}.
\end{align*}
\end{proof}
\begin{remark}
    While the two-level doubling trick works, for practical purposes we can set a small constant (e.g. $16$) to be an upper bound for $\ln{G_T}$ and perform a one-level doubling trick only on $\sum_{t=1}^T A_t$. This is because $\ln{G_T}$ increases exponentially slowly: if the upper bound for $\ln{G_T}$ is doubled in each increment starting from $2^0 = 1$, then to have $k$ such increments $G_T$ must be as large as $G_T \geq e^{2^k}$. For $k=5$, this is approximately $8 \times 10^{13}$, which is exceedingly large for the number of arms. Overall, this implies that the doubling level on $\ln{G_T}$ would change at most $4$ times in any practical scenario, and setting $\ln(G_T) \leq 16$ would contribute at most a multiplicative factor of $\sqrt{16} = 4$ on the total regret bound. 
\end{remark}

\section{FTARL WITH NEGATIVE ENTROPY IS EQUIVALENT TO SB-EXP3}
\label{appendix:SBEXP3ver2}
In Algorithm~\ref{algo:FTARL}, the loss estimate of non-active arms $a \notin \sA_t$ is $\tilde{\ell}_{a,t} = \hat{\ell}_t - \gamma\sum_{j \in \sA_t}\tilde{\ell}_{j,t}$. Because $I_{a,t} = 1$ for $a \in \sA_t$ and $I_{a,t}=0$ for $a \notin \sA_t$, it follows that in Algorithm~\ref{algo:FTARL}, for all $a \in [K]$,
\begin{align}
    \sum_{t=1}^T \left(\hat{\ell}_t - \gamma\sum_{j \in \sA_t}\tilde{\ell}_{j,t} - \tilde{\ell}_{a,t}\right) = \sum_{t=1}^T I_{a,t}\left(\hat{\ell}_t - \gamma\sum_{j \in \sA_t}\tilde{\ell}_{j,t} - \tilde{\ell}_{a,t}\right)
    \label{eq:notimeselection}
\end{align}
Next, we show that in each round, the sampling probability $p_t$ of FTARL with negative entropy is the same as that of SB-EXP3. With $\psi_t(x) = \frac{1}{\eta}\sum_{i=1}^K x_i\ln{x_i}$ the negative Shannon entropy, in Algorithm~\ref{algo:FTARL} the weight vector $q_t$ of FTARL is the solution of the optimization problem
\begin{align*}
    q_t = \min_{x \in \Delta_K} \frac{1}{\eta}\sum_{i=1}^K x_i\ln{x_i} + \sum_{i=1}^K x_i \tilde{L}_{i,t-1}.
\end{align*}
Solving for $q_t$, we obtain for any $i \in [K]$,
\begin{align*}
    q_{i,t} = \frac{\exp(-\eta \tilde{L}_{i,t-1})}{\sum_{k=1}^K \exp(-\eta \tilde{L}_{k,t-1})}.
\end{align*}
It follows that for an active arm $i \in \sA_t$,
\begin{nalign}
    p_{i,t} &= \frac{q_{i,t}}{\sum_{k \in \sA_t}q_{k,t}} \\
    &= \frac{\exp(-\eta \tilde{L}_{i,t-1})}{\sum_{k \in \sA_t}\exp(-\eta \tilde{L}_{k,t-1})} \\
    &= \frac{\exp(\eta\sum_{s=1}^{t-1}\hat{\ell}_s - \gamma\sum_{j \in \sA_s}\tilde{\ell}_{j,s} - \tilde{\ell}_{i,s} )}{\sum_{k \in \sA_t} \exp(\eta\sum_{s=1}^{t-1}\hat{\ell}_s - \gamma\sum_{j \in \sA_s}\tilde{\ell}_{j,s} -\tilde{\ell}_{k,s})} \\
    &= \frac{\exp(\eta\sum_{s=1}^{t-1}I_{i,s}(\hat{\ell}_s - \gamma\sum_{j \in \sA_s}\tilde{\ell}_{j,s} - \tilde{\ell}_{i,s}))}{\sum_{k \in \sA_t}\exp(\eta\sum_{s=1}^{t-1} I_{k,s}(\hat{\ell}_s - \gamma\sum_{j \in \sA_s}\tilde{\ell}_{j,s} - \tilde{\ell}_{k,s}))},
    \label{eq:pitFTARLnegativeentropy}
\end{nalign}
where
\begin{itemize}
    \item the second-to-last equality is by multiplying $\exp(\eta\sum_{s=1}^{t-1}\hat{\ell}_s - \gamma\sum_{j \in \sA_s}\tilde{\ell}_{j,s})$ to both the denominator and numerator.
    \item the last equality is due to~\eqref{eq:notimeselection}.
\end{itemize}
Observe that~\eqref{eq:pitFTARLnegativeentropy} is equal to the sampling probability of arm $i \in \sA_t$ computed in~\eqref{eq:pit} of Algorithm~\ref{sec:SB-EXP3}. Moreover, $p_{i,t} = 0$ for $i \notin \sA_t$ in both Algorithms~\ref{algo:SB-EXP3} and~\ref{algo:FTARL}. This implies that FTARL with negative Shannon entropy is equivalent to SB-EXP3.


\end{document}

%% file: math_commands.tex
\usepackage{amsmath,amsfonts,bm}

\def\1{\bm{1}}

\def\sqrtfrac#1#2{\sqrt{\frac{#1}{#2}}}

\DeclareMathAlphabet{\mathsfit}{\encodingdefault}{\sfdefault}{m}{sl}
\SetMathAlphabet{\mathsfit}{bold}{\encodingdefault}{\sfdefault}{bx}{n}

\def\gA{{\mathcal{A}}}

\def\gE{{\mathcal{E}}}

\def\gI{{\mathcal{I}}}

\def\sA{{\mathbb{A}}}
\def\sB{{\mathbb{B}}}

\def\sG{{\mathbb{G}}}

\def\sT{{\mathbb{T}}}

\def\sV{{\mathbb{V}}}

\newcommand{\E}{\mathbb{E}}

\newcommand{\R}{\mathbb{R}}

\newcommand{\I}[1]{\mathbb{1}\{#1\}}

\DeclareMathOperator*{\argmin}{arg\,min}

\DeclarePairedDelimiterX{\infdivx}[2]{(}{)}{%
  #1\;\delimsize\|\;#2%
}

\DeclarePairedDelimiterX{\inp}[2]{\langle}{\rangle}{#1, #2}

\newcounter{protocol}

\newenvironment{nalign}{
\allowdisplaybreaks
    \begin{equation}
    \begin{aligned}
}{
    \end{aligned}
    \end{equation}
    \ignorespacesafterend
}

%% file: sleepingbandits.bbl
\begin{thebibliography}{}

\bibitem[Abernethy et~al., 2015]{Abernethy2015Fighting}
Abernethy, J.~D., Lee, C., and Tewari, A. (2015).
\newblock Fighting bandits with a new kind of smoothness.
\newblock In {\em Advances in Neural Information Processing Systems},
  volume~28.

\bibitem[Adamskiy et~al., 2016]{Adamskiy2016}
Adamskiy, D., Koolen, W.~M., Chernov, A., and Vovk, V. (2016).
\newblock A closer look at adaptive regret.
\newblock {\em Journal of Machine Learning Research}, 17(23):1--21.

\bibitem[Audibert and Bubeck, 2009]{Audibert2009minimax}
Audibert, J.-Y. and Bubeck, S. (2009).
\newblock Minimax policies for adversarial and stochastic bandits.
\newblock In {\em Proceedings of the 22nd Annual Conference on Learning Theory
  (COLT)}.

\bibitem[Auer et~al., 2002]{EXP3Auer2002b}
Auer, P., Cesa-Bianchi, N., Freund, Y., and Schapire, R.~E. (2002).
\newblock The nonstochastic multiarmed bandit problem.
\newblock {\em SIAM Journal on Computing}, 32(1):48--77.

\bibitem[Blum, 1997]{Blum1997CalendarScheduling}
Blum, A. (1997).
\newblock Empirical support for winnow and weighted-majority algorithms:
  Results on a calendar scheduling domain.
\newblock {\em Machine Learning}, 26(1):5--23.

\bibitem[Blum and Mansour, 2007]{BlumAndMansour2007a}
Blum, A. and Mansour, Y. (2007).
\newblock From external to internal regret.
\newblock {\em Journal of Machine Learning Research}, 8:1307–1324.

\bibitem[Bouneffouf et~al., 2020]{Bouneffouf2020MABApplication}
Bouneffouf, D., Rish, I., and Aggarwal, C. (2020).
\newblock Survey on applications of multi-armed and contextual bandits.
\newblock In {\em 2020 IEEE Congress on Evolutionary Computation (CEC)}, pages
  1--8.

\bibitem[Chernov and Vovk, 2009]{ChernovVovk2009ExpertEvaluatorsAdvice}
Chernov, A. and Vovk, V. (2009).
\newblock Prediction with expert evaluators' advice.
\newblock In {\em Algorithmic Learning Theory}, pages 8--22, Berlin,
  Heidelberg.

\bibitem[Chernov and Vovk, 2010]{ChernovVovk2010UnknownNoExperts}
Chernov, A.~V. and Vovk, V. (2010).
\newblock Prediction with advice of unknown number of experts.
\newblock {\em CoRR}, abs/1006.0475.

\bibitem[Daniely et~al., 2015]{Daniely2015StronglyAdaptiveOL}
Daniely, A., Gonen, A., and Shalev-Shwartz, S. (2015).
\newblock Strongly adaptive online learning.
\newblock In {\em Proceedings of the 32nd International Conference on Machine
  Learning}, volume~37, pages 1405--1411, Lille, France.

\bibitem[Freund and Schapire, 1997]{Freund1997}
Freund, Y. and Schapire, R.~E. (1997).
\newblock {A Decision-Theoretic Generalization of On-Line Learning and an
  Application to Boosting}.
\newblock {\em Journal of Computer and System Sciences}, 55(1):119--139.

\bibitem[Freund et~al., 1997]{Freund1997SpecializedExperts}
Freund, Y., Schapire, R.~E., Singer, Y., and Warmuth, M.~K. (1997).
\newblock Using and combining predictors that specialize.
\newblock In {\em Proceedings of the Twenty-Ninth Annual ACM Symposium on
  Theory of Computing}, STOC '97, page 334–343, New York, NY, USA.

\bibitem[Gaillard et~al., 2023]{Gaillard2023OneArrowTwoKills}
Gaillard, P., Saha, A., and Dan, S. (2023).
\newblock One arrow, two kills: A unified framework for achieving optimal
  regret guarantees in sleeping bandits.
\newblock In {\em Proceedings of The 26th International Conference on
  Artificial Intelligence and Statistics}, volume 206, pages 7755--7773.

\bibitem[Gaillard et~al., 2014]{Gaillard2014SecondOrderBoundExcessLosses}
Gaillard, P., Stoltz, G., and Erven, T. (2014).
\newblock A second-order bound with excess losses.
\newblock {\em Journal of Machine Learning Research}, 35.

\bibitem[Hazan and Seshadhri, 2009]{HazanAdaptiveRegret2009}
Hazan, E. and Seshadhri, C. (2009).
\newblock Efficient learning algorithms for changing environments.
\newblock In {\em Proceedings of the 26th Annual International Conference on
  Machine Learning}, ICML '09, page 393–400, New York, NY, USA.

\bibitem[Herbster and Warmuth, 1998]{Herbster1998FixedShare}
Herbster, M. and Warmuth, M.~K. (1998).
\newblock Tracking the best expert.
\newblock {\em Machine Learning}, 32(2):151--178.

\bibitem[Kanade and Steinke, 2014]{Kanade2014SleepingExperts}
Kanade, V. and Steinke, T. (2014).
\newblock Learning hurdles for sleeping experts.
\newblock {\em ACM Trans. Comput. Theory}, 6(3).

\bibitem[Kleinberg et~al., 2010]{Kleinberg2010Sleeping}
Kleinberg, R., Niculescu-Mizil, A., and Sharma, Y. (2010).
\newblock Regret bounds for sleeping experts and bandits.
\newblock {\em Machine Learning}, 80(2–3):245–272.

\bibitem[Lattimore and Szepesvári, 2020]{BanditAlgorithmsBook2020}
Lattimore, T. and Szepesvári, C. (2020).
\newblock {\em Bandit Algorithms}.
\newblock Cambridge University Press.

\bibitem[Luo, 2017]{Luo2017CSCI699LectureNote13}
Luo, H. (2017).
\newblock Lecture 13, {I}ntroduction to {O}nline {L}earning.
\newblock \url{https://haipeng-luo.net/courses/CSCI699/lecture13.pdf}.

\bibitem[Luo and Schapire, 2015]{Luo2015AdaNormalHedge}
Luo, H. and Schapire, R.~E. (2015).
\newblock Achieving all with no parameters: Ada{N}ormal{H}edge.
\newblock In {\em Annual Conference Computational Learning Theory}.

\bibitem[Luo et~al., 2018]{Luo2018ContextualNonstationary}
Luo, H., Wei, C.-Y., Agarwal, A., and Langford, J. (2018).
\newblock Efficient contextual bandits in non-stationary worlds.
\newblock In Bubeck, S., Perchet, V., and Rigollet, P., editors, {\em
  Proceedings of the 31st Conference On Learning Theory}, volume~75 of {\em
  Proceedings of Machine Learning Research}, pages 1739--1776. PMLR.

\bibitem[Neu, 2015]{Neu2015ExploreNM}
Neu, G. (2015).
\newblock Explore no more: Improved high-probability regret bounds for
  non-stochastic bandits.
\newblock In {\em Advances in Neural Information Processing Systems},
  volume~28.

\bibitem[Neu and Valko, 2014]{Neu2014CombinatorialSleepingPolicyRegret}
Neu, G. and Valko, M. (2014).
\newblock Online combinatorial optimization with stochastic decision sets and
  adversarial losses.
\newblock In {\em Advances in Neural Information Processing Systems},
  volume~27.

\bibitem[Orabona, 2019]{OrabonaIntroToOnlineLearningBook}
Orabona, F. (2019).
\newblock A modern introduction to online learning.
\newblock {\em CoRR}, abs/1912.13213.

\bibitem[Saha et~al., 2020]{Saha2020}
Saha, A., Gaillard, P., and Valko, M. (2020).
\newblock Improved sleeping bandits with stochastic actions sets and
  adversarial rewards.
\newblock In {\em Proceedings of the 37th International Conference on Machine
  Learning}, ICML'20.

\bibitem[Slivkins, 2013]{Slivkins2013}
Slivkins, A. (2013).
\newblock Dynamic ad allocation: Bandits with budgets.
\newblock {\em CoRR}, abs/1306.0155.

\bibitem[Slivkins, 2014]{Slivkins2014}
Slivkins, A. (2014).
\newblock Contextual bandits with similarity information.
\newblock {\em Journal of Machine Learning Research}, 15(1):2533–2568.

\end{thebibliography}
